\documentclass[runningheads]{llncs}
\usepackage[T1]{fontenc}
\usepackage{amsmath,amssymb}
\usepackage{tikz}
\usepackage{bm,xspace}
\usetikzlibrary{trees, shapes}
\usepackage{graphicx}
\usepackage{booktabs}
\usepackage[noend, linesnumbered]{algorithm2e}
\usepackage{multirow}

\usetikzlibrary{arrows.meta,positioning}
\usepackage{pgfplots}
\pgfplotsset{compat=1.18} 

\newcommand{\AP}{\mathcal{AP}}
\newcommand{\apex}{\bm{A}}
\newcommand{\EPS}{\bm{\varepsilon}}
\newcommand{\epsdom}{\preceq_{\bm{\varepsilon}}}
\newcommand{\open}{\ensuremath{\textsc{Open}}\xspace}

\newcommand{\RP}{\mathcal{RP}}

\newcommand\algname[1]{\textsf{#1}\xspace}
\newcommand\apexAlg{\algname{A*pex}}
\newcommand\rapexAlg{\algname{RA*pex}}
\newcommand\rbexAlg{\algname{RB-Exact}}
\newcommand\rbapAlg{\algname{RB-Approx}}

\newcommand{\ignore}[1]{}

\begin{document}
\title{Approximate Multi-Objective Search Under Rulebooks}
%
%
\author{Omar Muhammetkulyyev \inst{1} \and
Oren Salzman \inst{2} \and
Tichakorn Wongpiromsarn \inst{1}}
\authorrunning{O. Muhammetkulyyev et al.}
%
\institute{Iowa State University, Ames IA 50010, USA \\ \email{omar99@iastate.edu}, \email{nok@iastate} \and
Technion - Israel Institute of Technology, Haifa 3200003, Israel
\email{osalzman@cs.technion.ac.il}
}
\maketitle              
\vspace{-10mm}
\begin{abstract}
Robotic planning often involves multiple objectives with complex priority relationships, such as safety, efficiency, and regulatory compliance. Rulebooks formalize these relationships, allowing partial ordering of objectives that generalizes both Pareto and lexicographic dominance. Computing the full set of rulebook-optimal solutions, however, is computationally expensive. To address this challenge, we introduce the concept of $\EPS$-rule-dominance, a principled notion of approximate dominance under rulebooks, and propose \rapexAlg, a best-first search algorithm that efficiently computes a compact set of $\EPS$-approximate rulebook-optimal solutions. \rapexAlg leverages dimensionality reduction, a technique used to speed up existing multi-objective search algorithms,  while respecting rule hierarchies by maintaining separate closed sets and performing dominance checks over truncated and residual rule sets. We provide a formal analysis of \rapexAlg, proving that every rulebook-optimal solution is $\EPS$-rule-dominated (a generalization of approximate dominance we introduce) by at least one solution in the returned set. Empirical results demonstrate that our approach achieves computation times over two orders of magnitude faster than existing methods.

\vspace{-2mm}
\end{abstract}
\vspace{-10mm}
\section{Introduction}
\vspace{-3mm}
Robotic planning and navigation problems often require reasoning over multiple, potentially conflicting objectives, including safety, efficiency, and task-related constraints. For example, an autonomous vehicle must avoid collisions, respect traffic rules, and minimize travel time and energy consumption. These requirements naturally lead to multi-objective search (MOS), in which planning is performed on graphs whose edge weights are represented by vectors of multiple cost components, each representing a distinct objective.

In classical MOS, solution paths are compared using Pareto dominance, i.e., a path $\pi$ dominates another path $\pi'$ if every cost component of $\pi$ is no larger than the corresponding component of $\pi'$ and at least one component is strictly smaller. Instead of returning a single optimal solution, MOS aims to compute the complete set of non-dominated solution paths. While this formulation captures trade-offs among objectives without requiring manual weighting, it is computationally demanding as the number of non-dominated paths grows exponentially with the size of the graph~\cite{breugem2017analysis,Ehrgott05}. Moreover, unlike a single objective search where there is only one optimal cost, the number of optimal cost vectors in MOS are also exponential in terms of the graph size.

The inherent complexity of MOS has motivated the development of approximate algorithms \cite{GS21,zhang2022pex}. Instead of computing the complete set of non-dominated solutions, these methods aim to compute a reduced representative subset that approximates the Pareto frontier within a specified tolerance.

Many robotic applications, however, impose additional structure on objective relationships that is not captured by Pareto dominance alone~\cite{PWAS25}. In particular, safety-critical objectives are often strictly more important than efficiency-related objectives. This observation has led to lexicographic formulations that impose a strict priority ordering among objectives \cite{SlutskyYWF21}. Lexicographic search can substantially reduce the solution space, but it requires a strict total ordering and cannot handle partial or mixed priority relationships.

In practice, robotic planning problems frequently involve a combination of hierarchical and non-comparable objectives. For example, collision avoidance may be strictly prioritized over regulatory compliance, while objectives such as minimizing energy consumption and travel time may remain mutually non-comparable. We adopt the \emph{rulebooks} formalism \cite{CensiSWYPFF19} to formalize complex relationships among objectives. Rulebooks capture partial priority structures that subsume both Pareto dominance and lexicographic ordering. However, computing the set of non-dominated solutions under a rulebook remains computationally expensive and exhibits exponential worst-case complexity \cite{WongpiromsarnSF26}.

This paper addresses the computational challenges of MOS under rulebooks. Our contributions are threefold. First, we introduce the notion of $\EPS$-dominance with respect to rulebooks, which provides a principled approximation of the rulebook non-dominated solution set. Second, we propose an algorithm that computes a set of approximate solutions such that every solution path is $\EPS$-rule-dominated by at least one path in the returned set. Finally, we present extensive experimental results demonstrating that our approach achieves significant computational improvements (up to two orders of magnitude faster) over existing algorithms that compute the complete set of rulebook non-dominated solutions, while preserving the intended objective priorities.

\vspace{-7mm}
\section{Preliminaries}
\vspace{-6mm}
This section presents formal definitions and concepts of MOS~\cite{SalzmanF0ZCK23} and rulebooks~\cite{CensiSWYPFF19,WongpiromsarnSF26}. Throughout the paper, we let $\mathbb{R}, \mathbb{R}_{\geq 0}, \mathbb{R}_{> 0}$ and $\mathbb{N}$ denote the set of real, non-negative real, positive real, and natural numbers, respectively. Boldface symbols denote vectors or vector-valued functions and $v_i$ denotes the $i$-th component of a vector or vector function $\boldsymbol{v}$.

Let $\mathbf{p}$ and~$\mathbf{q}$ be $N$-dimensional vectors.
We write $\mathbf{p}+\mathbf{q}$ to denote component-wise addition.
For a minimization problem, we say that $\mathbf{p}$ \emph{weakly dominates} $\mathbf{q}$,
denoted $\mathbf{p} \preceq \mathbf{q}$, if $p_i \leq q_i$ for all $i$.
We say that $\mathbf{p}$ \emph{dominates} $\mathbf{q}$, denoted $\mathbf{p} \prec \mathbf{q}$,
if $\mathbf{p} \preceq \mathbf{q}$ and $p_j < q_j$ for at least one index $j$.
%
%
Finally,
let~$\EPS \in \mathbb{R}_{\geq 0}^N$ be another non-negative $N$-dimensional vector.
We say that~$\mathbf{p}$ \emph{approximately dominates}~$\mathbf{q}$ with an \emph{approximation factor} $\EPS$, denoted $\mathbf{p} \preceq_{\EPS} \mathbf{q}$, if $p_i \leq (1+\varepsilon_i) q_i$ for all $i$.

\vspace{-3mm}
\subsection{Multi-Objective Search (MOS)}
\vspace{-1mm}
In Multi-Objective Search (MOS), we are given a directed graph~$G=(V,E)$ where each edge $e\in E$ has a nonnegative cost vector $\mathbf{c}(e)\in\mathbb{R}_{\geq 0}^N$, where $N>0$ is the number of objectives.
A path is a sequence of vertices $\pi=\langle v_1, \ldots, v_k\rangle$ such that $v_i \in V$ and $(v_i, v_{i+1}) \in E$ for all $i$.
We let $\Pi$ denote the set of all paths in $G$.
The cost of a path $\pi \in \Pi$ is defined as the component-wise sum of the cost vectors of its edges: $c(\pi) = \sum_i \mathbf{c}(v_i, v_{i+1})$.

Given start and target vertices $s,t \in V$,
a path from $s$ to $t$ is called a \emph{solution}.
A solution is  \emph{Pareto-optimal} iff its cost is not dominated by any other solution.
The objective of the basic MOS problem is to compute the set $\Pi^\star \subseteq \Pi$ of Pareto-optimal solutions, also known as the \emph{Pareto front} (PF), for given $s, t \in V$.
As the size of $\Pi^\star$ may be exponential in $\vert V \vert$, computing the entire PF is often impractical.
Thus, we instead seek its bounded approximation.
%
Specifically, given an approximation factor~$\EPS \in \mathbb{R}_{\geq 0}^N$,
the \emph{$\EPS$-approximate PF}, denoted $\Pi^\star_{\EPS}$, is a set of solutions such that
$\forall \pi \in \Pi^\star,
\exists \pi' \in \Pi^\star_{\EPS}
~s.t.~
\mathbf{c}(\pi') \preceq_{\EPS} \mathbf{c}(\pi)$.
Namely, every solution in $\Pi^\star$ is approximately dominated by some solution in~$\Pi^*_{\EPS}$.

\vspace{-5mm}
\subsubsection{Exact MOS algorithms and dimensionality reduction.}
Arguably, the most common approach to MOS is using a best-first search approach~\cite{AhmadiTHK21,CasasSB21,CasasKSB23,de2005new,mandow2010multiobjective,ren2025emoa}
which generalizes the celebrated \algname{A$^*$} algorithm~\cite{HNR68} to the MOS setting.

A key insight that dramatically improved the efficiency of these algorithms was to order the nodes in the priority queue in increasing lexicographic order and apply the notion of \emph{dimensionality
reduction}~\cite{pulido-2015-namoadr}.
More specifically, when the search frontier is ordered lexicographically, the value of the first objective is guaranteed to be non-decreasing throughout the expansion process. This monotonicity implies that any newly generated node will essentially have a first-objective cost that is equal to or greater than that of any previously expanded node, rendering explicit comparisons for this dimension redundant. As a result, the dominance check can be restricted to the remaining objectives, effectively reducing the dimensionality of the problem by one.
This was shown (see, e.g.,~\cite{hernandez2023simple,pulido-2015-namoadr}) to have a dramatic impact on the running time of the algorithm. For additional details, see~\cite{pulido-2015-namoadr,SalzmanF0ZCK23}.

\vspace{-5mm}
\subsubsection{Approximate MOS algorithms.}
Arguably, the state-of-the-art algorithm to compute an $\EPS$-approximate PF is \apexAlg~\cite{zhang2022pex}, which also serves as the algorithmic foundation of our algorithm.
%
Similar to exact MOS heuristic search algorithms, \apexAlg performs a best-first search over the search space. The efficiency of \apexAlg stems from the fact that instead of reasoning about single paths, \apexAlg reasons about \emph{sets} of paths with the same last vertex and similar costs, which results in small numbers of search-node expansions and thus small runtimes.

Specifically, each node of \apexAlg is represented by an \emph{apex-path pair} $\AP=\langle \apex, \pi \rangle$, where $\apex \in \mathbb{R}_{\geq 0}^N$ is called the node's \emph{apex} and $\pi \in \Pi$ is called  the node's \emph{representative path}, with the requirement that $\apex \preceq \mathbf{c}(\pi)$, i.e., the apex weakly dominates the cost of the representative path.
%
%
The vertex $v(\AP)$ of $\AP$ is defined as $v(\AP):=v(\pi)$, where $v(\pi) \in V$ is the last vertex of path $\pi$.
Each apex-path pair $\AP=\langle \apex,\pi \rangle$ corresponds to a set of paths $\Pi_\AP$ (which includes~$\pi$) with the same last vertex $v(\AP)$.
Instead of storing $\Pi_\AP$ explicitly, \apexAlg
only stores the apex $\apex = \min_{\pi'\in \Pi_\AP}\{ \mathbf{c}(\pi')\}$, which is the component-wise minimum of (and hence weakly dominates) the costs of all paths in $\Pi_\AP$, including the representative path $\pi \in \Pi_\AP$.
Given a heuristic function $\mathbf{h} : V \to \mathbb{R}_{\geq 0}^N$,
the $\mathbf{g}$-value and $\mathbf{f}$-value of $\AP=\langle \apex, \pi \rangle$ are defined as
$\mathbf{g}(\AP):=\apex$ and $\mathbf{f}(\AP):=\apex + \mathbf{h}(v(\AP))$.
Finally, we say that $\AP$ is $\EPS$-bounded iff~$\mathbf{f}(\pi) \epsdom \mathbf{f}(\AP)$, where $\mathbf{f}(\pi) := \mathbf{c}(\pi) + \mathbf{h}(v(\pi))$.
As we will see, all nodes in \apexAlg will be $\EPS$-bounded.

Before we explain how apex-path pairs are used, let us define the operations on apex-path pairs:
First, \emph{extending}  an apex-path pair $\AP=\langle \apex, \pi \rangle$ by an edge $e$ results in an apex-path pair $\AP' = \langle \apex + \mathbf{c}(e), \pi' \rangle$, where $\pi'$ is a path $\pi$ extended by edge $e$.
Conceptually, apex-path pair $\AP'$ corresponds to the set of paths~$\Pi_{\AP'}$ that extends every path in $\Pi_\AP$ by $e$. It is easy to verify that the apex of~$\AP'$ is the component-wise minimum of the costs of the paths in $\Pi_{\AP'}$.

The second operation is \textit{merging} two apex-path pairs that contain the same vertex.
Conceptually, merging two apex-path pairs corresponds to merging the two sets of paths that these two apex-path pairs correspond to.
Hence, the apex of the merged apex-path pair is the component-wise minimum of the apexes of the two apex-path pairs.
The representative path of the merged apex-path pair is either one of the two representative paths of the two apex-path pairs.

\apexAlg starts with a single apex-path pair $\langle \mathbf{0}, [s] \rangle$ in a priority-queue \open.
In each iteration, \apexAlg extracts an apex-path pair $\AP$ from \open with the lexicographically smallest $\mathbf{f}$-value. The algorithm then performs dominance checks and discards $\AP$ if either
(i)~there exists a solution already found whose cost $\EPS$-dominates the $\mathbf{f}$-value of $\AP$
or if
(ii)~there exists an expanded apex-path pair that contains $v(\AP)$ and whose $\mathbf{g}$-value weakly dominates $\mathbf{g}(\AP)$

When \apexAlg expands an apex-path pair that contains the goal vertex $t$, it adds the representative path of this apex-path pair to the set of found solutions.
When \apexAlg expands an apex-path pair that does not contain $t$, it generates a child apex-path pair~$\AP$ for each vertex $v'$ such that $(v,v') \in E$  by extending the expanded apex-path pair with edge $(v,v')$.
Consider that the child apex-path pair $\AP$ is not discarded after the dominance checks, and let $\open[v]$ be the set of apex-path pairs in \open that contains vertex~$v$. \apexAlg then checks   if there exists an apex-path pair $\AP'$ in $\open[v(\AP)]$ that results in an $\EPS$-bounded apex-path pair when merged with $\AP$. If so, \apexAlg removes~$\AP'$ from \open and then adds the merged apex-path pair to \open. Otherwise, it adds $\AP$ to \open.
When \open becomes empty, \apexAlg terminates and returns the set of solutions as an $\EPS$-approximate Pareto frontier.

\vspace{-4mm}
\subsection{Rulebooks}
\label{ssec:rulebooks}
\vspace{-2.5mm}
To avoid ambiguity in terminology, we fix the following conventions for the remainder of the paper. Specifically, we follow standard robotics terminology and use the term \emph{state} to refer to what is commonly called a \emph{vertex} in graph-search literature, and \emph{realization} to refer to a \emph{path}. Readers familiar with graph terminology may interpret these terms interchangeably.

A rulebook is evaluated over a set of possible outcomes, referred to as the set of \emph{realizations} and denoted by $\Sigma$. In the context of multi-objective search, a realization corresponds to a path in the underlying graph. A rulebook consists of two main components: a set of rules and a preorder that specifies their relative importance. Each rule corresponds to a distinct objective and assigns a non-negative cost that reflects the degree to which the realization violates the objective. The preorder encodes priority relationships among rules, allowing the representation of both hierarchical objectives and objectives that are not comparable.

\vspace{-2mm}
\begin{definition}
  A  \emph{rule} is a function $r : \Sigma \to \mathbb{R}_{\geq 0}$ that measures the
  degree of violation of its argument.
\end{definition}
\vspace{-2mm}

For any realizations $x, y \in \Sigma$, $r(x) < r(y)$ indicates that $y$ violates the rule $r$ to a greater extent than $x$. In particular, $r(x) = 0$ indicates that $x$ fully satisfies the rule. Although we refer to these functions as ``rules'', they are not limited to regulatory constraint. A rule may represent performance-related objectives like efficiency or user preferences that are desirable but not strictly required.

\vspace{-2mm}
\begin{definition}
  A \emph{preorder} on a set $S$ is a binary relation $\lesssim$ that is
  reflexive ($s \lesssim s$ for all $s \in S$), and
  transitive ($s_{1} \lesssim s_{2}$ and $s_{2} \lesssim s_{3}$ imply $s_{1} \lesssim s_{3}$ for
  all $s_{1}, s_{2}, s_{3} \in S$).
\end{definition}
\vspace{-2mm}

A preorder generalizes a partial order by relaxing the antisymmetry requirement. In particular, it allows both $s_{1} \lesssim s_{2}$ and $s_{2} \lesssim s_{1}$ to hold for distinct elements $s_{1} \not= s_{2}$. This property enables the representation of objectives that are of equal priority \cite{WongpiromsarnSF26}. One can define an equivalence relation $\sim$ on $S$ so that $s_1 \sim s_2$ if and only if $s_1 \lesssim s_2$ and $s_2 \lesssim s_1$. The preorder on $S$ can then be viewed as a partial order on $S \, \setminus \sim$ where each element is an equivalence class of $\sim$.

\vspace{-2mm}
\begin{definition}
  A rulebook is defined as a tuple $\mathcal{R} = \langle R, \lesssim \rangle$,
  where $R$ is the set of rules and $\lesssim$ is a preorder on $R$
  that specifies their relative importance.
\end{definition}
\vspace{-2mm}


As proved by Slutsky et al.~\cite{CensiSWYPFF19}, a rulebook $\mathcal{R}$ induces a preorder $\lesssim_{\mathcal{R}}$ on $\Sigma$, ensuring consistency and preventing cyclic preferences among realizations. Intuitively, a realization $x$ is considered ``at least as desirable as'' a realization $y$ if any increased violation in a lower-priority rule is compensated by an improvement in a higher-priority rule. Formally, we say that $x$ \emph{weakly rule-dominates} $y$, denoted $x \lesssim_{\mathcal{R}} y$, if for any rule $r' \in R$ such that $r'(x) > r'(y)$, there exists a higher priority rule $r > r'$ such that $r(x) < r(y)$. We say that $x$ \emph{rule-dominates}~$y$, denoted by $x <_{\mathcal{R}} y$, if $x \lesssim_{\mathcal{R}} y$ but $y \not\lesssim_{\mathcal{R}} x$.




\vspace{-4mm}
\section{Problem Formulation}
\vspace{-3mm}

We consider multi-objective search on a graph $G = (S, E)$, where $S$ is a finite set of states and $E \subseteq S \times S$ is a finite set of edges. A realization is any path in $G$, regardless of its start or end states. We let $\Sigma$ denote the set of all realizations. The quality of a realization is evaluated using a rulebook $\mathcal{R} = \langle R, \lesssim \rangle$ where $R = \{r_1, \ldots, r_N\}$, which induces a preorder $\lesssim_{\mathcal{R}}$ on $\Sigma$ as defined in Section \ref{ssec:rulebooks}.

Let $s_{\text{start}} \in S$ and $s_{\text{goal}} \in S$ denote the start and goal states, respectively. We define the set of solutions as the subset of realizations that start at $s_{\text{start}}$ and end at $s_{\text{goal}}$ and denote this set by $\Sigma_{\text{sol}} \subseteq \Sigma$.

We first define the notion of optimality with respect to a rulebook. A solution is said to be \emph{rulebook-optimal} if it is not rule-dominated by any other solution. The set of all such solutions forms the \emph{rulebook-optimal solution set}.

\vspace{-1mm}
\begin{definition}
The \emph{rulebook-optimal solution set} is defined as $\mathcal{P}_{\mathcal{R}} = \{x \in \Sigma_{\text{sol}} \ | \ \not\exists y \in \Sigma_{\text{sol}} \text{ such that } y <_{\mathcal{R}} x\}$
\end{definition}
\vspace{-2mm}

The set $\mathcal{P}_{\mathcal{R}}$ generalizes the classical Pareto-optimal set and captures both hierarchical and non-comparable objectives. As in standard multi-objective search, the size of $\mathcal{P}_{\mathcal{R}}$ can grow exponentially in the size of the graph, making exact computation impractical for large problems. To enable approximation of $\mathcal{P}_{\mathcal{R}}$, we introduce a relaxed notion of rule-dominance.

\vspace{-2mm}
\begin{definition}
Let $\mathcal{R} = \langle R, \lesssim \rangle$ be a rulebook, where $R = \{r_1, \ldots, r_N\}$, and $\EPS \in \mathbb{R}_{\geq 0}^{N}$. For realizations $x, y \in \Sigma$, we say that $x$ \emph{$\EPS$-rule-dominates} $y$, denoted $x \lesssim_{\mathcal{R}}^{\EPS} y$ if for any rule $r_j \in R$ such that $r_j(x) > (1+\varepsilon_{j}) r_j(y)$, there exists a higher-priority rule $r_i > r_j$ such that $r_i(x) < (1+\varepsilon_{i}) r_i(y)$.
\label{def:eps-rule-dom}
\end{definition}
\vspace{-4mm}

\vspace{-2mm}
\begin{example}
Consider a rulebook $\mathcal{R} = \langle R, \lesssim \rangle$ with $R = \{r_1, r_2, r_3\}$, $r_1 > r_2$ and $r_1 > r_3$, i.e., $r_1$ is the highest priority rule while $r_2$ and $r_3$ are not comparable.
Consider two realizations $x, y \in \Sigma$ with
$r_1(x) = 1.9$, $r_2(x) = r_3(x) = 2$ and $r_1(y) = r_2(y) = r_3(y) = 1$.
Then, $y \lesssim_{\mathcal{R}} x$ but $x \not\lesssim_{\mathcal{R}} y$ because $r_1(x) > r_1(y)$ and there is no rule $r_i$ satisfying $r_i > r_1$. However, with $\EPS = \mathbf{1}$, we get  $y \lesssim_{\mathcal{R}}^{\EPS} x$ and $x \lesssim_{\mathcal{R}}^{\EPS} y$ because $r_1(x) < (1+\varepsilon_1) r_1(y)$.
\label{ex:eps-rule-dominance}
\end{example}
\vspace{-2mm}


This definition relaxes rule-dominance by allowing bounded degradation, controlled by $\EPS$. Using $\EPS$-rule-dominance, we define a notion of  an approximate optimal solution set.

\vspace{-2mm}
\begin{definition}
A set $\mathcal{P}_{\mathcal{R}}^{\EPS} \subseteq \Sigma_{\text{sol}}$ is an \emph{$\EPS$-approximate rulebook-optimal solution set}  if for every rulebook-optimal solution $x' \in \mathcal{P}_{\mathcal{R}}$, there exists a solution $x \in \mathcal{P}_{\mathcal{R}}^{\EPS}$ such that $x \lesssim_{\mathcal{R}}^{\EPS} x'$.
\end{definition}
\vspace{-2mm}

We note that this definition does not require the elements of $\mathcal{P}_{\mathcal{R}}^{\EPS}$ to be rulebook-optimal or even near-optimal. In the extreme case, the entire solution set $\Sigma_{\text{sol}}$ trivially satisfies this definition. Our goal is not to minimize the suboptimality of the returned solutions, but to efficiently compute $\mathcal{P}_{\mathcal{R}}^{\EPS}$, ideally one of small size.

\vspace{-2mm}
\begin{problem}[Approximate Multi-Objective Search under Rulebooks:]
Given a graph $G = (S, E)$, start and goal states $s_{\text{start}}, s_{\text{goal}} \in S$, a rulebook $\mathcal{R} = \langle R = \{r_1, \ldots, r_N\}, \lesssim \rangle$, and a tolerance vector $\EPS \in \mathbb{R}_{\geq 0}^N$ , compute an $\EPS$-approximate rulebook-optimal solution set $\mathcal{P}_{\mathcal{R}}^{\EPS}$.
\label{prob:rapprox}
\end{problem}

\vspace{-8mm}
\section{Rulebook-{\apexAlg} (\rapexAlg)}
\vspace{-3mm}
This section presents rulebook-\apexAlg (\rapexAlg), our algorithm for Problem~\ref{prob:rapprox}.
At a high level, \rapexAlg follows the same best-first ``expand-and-prune'' template as \apexAlg, but replaces component-wise $\EPS$-dominance $\preceq_{\EPS}$ with the rulebook-based relation $\lesssim_{\mathcal{R}}^{\EPS}$.
This change is conceptually simple, yet it has two important consequences that drive the technical development in the rest of the section.

\apexAlg can rely on purely component-wise reasoning: when \open is ordered lexicographically by $\mathbf{f}$, the first component becomes monotone along the search and can be ignored in dominance checks (dimensionality reduction).
Under rulebooks, dominance depends on the \emph{priority structure} among rules and, crucially, on whether a difference is \emph{strict} (an improvement in a higher-priority rule can compensate for a degradation in a lower-priority rule).
As a result, the simple component-wise dominance machinery of \apexAlg no longer applies directly.

We first give a baseline \rapexAlg that operates on full rule-value vectors and uses $\lesssim_{\mathcal{R}}$ and $\lesssim_{\mathcal{R}}^{\EPS}$ directly for pruning.
This baseline introduces the algorithmic logic but does not benefit from dimensionality reduction.
We then develop a dimensionality-reduction variant: we explain why the \apexAlg argument fails under rulebooks, and show how to restore an efficient dominance test by partitioning expanded nodes  while preserving correctness.

\ignore{
To solve Problem \ref{prob:rapprox}, we extend the core idea of \apexAlg to support $\varepsilon$-rule-dominance. Similar to \apexAlg, we seek to compute a small representative set of solutions that constitutes an $\EPS$-approximate rulebook-optimal solution. This extension, however, is non-trivial because in \apexAlg, $\EPS$-dominance is defined component-wise. This structure enables a simple form of dimensionality reduction: by ordering \open lexicographically by the $\mathbf{f}$-values, the first component is implicitly handled, and dominance checks only need to compare the remaining components.

This idea does not carry over to rulebook dominance because the hierarchical structure of rules prevents lexicographic ordering from isolating a single component and thus complicate dominance checks. To highlight these challenges, we first present a baseline version of our algorithm, rulebook-\apexAlg (\rapexAlg) that operates directly on full rule-value vectors, without dimensionality reduction. We then show how to incorporate dimensionality reduction to recover much of the efficiency of \apexAlg while preserving correctness under $\EPS$-rule-dominance.}

\vspace{-4mm}
\subsection{Baseline \rapexAlg}
\vspace{-2mm}
To simplify notation, we define a mapping that associates each realization with its vector representation of rule-violation values.

\vspace{-2.5mm}
\begin{definition}
Let $\mathcal{R} = \langle R, \lesssim \rangle$ be a rulebook with $R = \{r_1, \ldots, r_N\}$. We define the mapping $\mathbf{c}_{\mathcal{R}} : \Sigma \to \mathbb{R}_{\geq 0}^N$ by $\mathbf{c}_{\mathcal{R}}(x) = [r_1(x), \ldots, r_N(x)]$ so that the $i$-th component of $\mathbf{c}_{\mathcal{R}}(x)$ represents the degree to which realization $x$ violates rule $r_i$.
\end{definition}
\vspace{-2mm}

This vector representation allows us to reason directly about $N$-dimensional vectors (similar to the way that \apexAlg reasons about path costs).
In particular, both rule-dominance and $\EPS$-rule-dominance can be defined and evaluated purely at the vector level, without reference to the underlying realization.

\vspace{-2.5mm}
\begin{definition}
Let $\mathcal{R} = \langle R, \lesssim \rangle$ be a rulebook with $R = \{r_1, \ldots, r_N\}$ and let
$\mathbf{u}, \mathbf{v} \in \mathbb{R}_{\geq 0}^N$ be $N$-dimensional vectors.
We say that $\mathbf{u} \lesssim_{\mathcal{R}} \mathbf{v}$, if for any rule $r_j \in R$ such that
$u_j > v_j$, there exists a higher-priority rule $r_i > r_j$ such that $u_i < v_i$.
\end{definition}
\vspace{-3mm}




Similarly, $\EPS$-rule-dominance $\lesssim_{\mathcal{R}}^{\EPS}$ can be defined for vectors in the same way as for realizations in Definition \ref{def:eps-rule-dom}. Motivated by \apexAlg, we represent each node by a \emph{rule-apex-realization pair}.

\vspace{-2.5mm}
\begin{definition}
    A \emph{rule-apex-realization pair} $\RP = \langle \mathbf{r},x \rangle$
    consists of a
    \emph{rule apex} $\mathbf{r} \in \mathbb{R}_{\geq 0}^N$
    and a
    \emph{representative realization}  $\mathbf{r} \lesssim_{\mathcal{R}}\mathbf{c}_{\mathcal{R}}(x)$.
\end{definition}
\vspace{-2mm}

A state of $\RP = \langle \mathbf{r},x \rangle$, denoted $s(\RP)$, is defined as the last state of $x$. Similar to \apexAlg, for a given heuristic function $\mathbf{h} : S \to \mathbb{R}_{\geq 0}^N$, we define the $\mathbf{g}$-value and $\mathbf{f}$-value of $\RP$ as $\mathbf{g}(\RP) := \mathbf{r}$ and $\mathbf{f}(\RP) := \mathbf{r} + \mathbf{h}(s(\RP))$. We say that $\RP$ is $\EPS$-bounded iff $\mathbf{f}(x) \lesssim_{\mathcal{R}}^\varepsilon \mathbf{f}(\RP)$, where $\mathbf{f}(x) := \mathbf{c}_{\mathcal{R}}(x) + \mathbf{h}(s(x))$. \rapexAlg ensures that all the rule-apex-realization pairs are $\EPS$-bounded.


\setlength{\floatsep}{2pt} 
\SetAlgoSkip{}
\SetKwComment{Comment}{/* }{ */}
\SetKwInput{KwInput}{Input}
\RestyleAlgo{ruled}
\SetKwIF{If}{ElseIf}{Else}{if}{}{else if}{else}{endif}
\SetKw{KwThen}{then}
\begin{algorithm}[t]
$\open \gets \{ \langle$ $\mathbf{0}$, $ [s_{\text{start}}] \rangle \}$
\tcp*[r]{ordered according to rule-dominance $\lesssim_{\mathcal{R}}$}
\label{line:rapex:init-open}
$solutions \gets \emptyset$\;\label{line:rapex:init-solutions}
\For{$s \in S$}{
    $G_{\rm{cl}}(s) \gets \emptyset$ \;\label{line:rapex:init-gcl}
}
\While{$\open \not = \emptyset$}{
    $\mathcal{RP} = \langle \mathbf{r}, x\rangle \gets \open.\mathsf{extract\_min()}$  \label{line:rapex:extract}

    \If{$\mathsf{is\_dominated}(\mathcal{RP}, G_{\rm{cl}}, solutions)$ \KwThen \tcp*[f]{Alg.~\ref{alg:is_dominated}}}{ \label{line:rapex:prune}
        $\mathbf{continue}$ \;
    }

    $G_{\rm{cl}}(s(\mathcal{RP})).add(\mathbf{f}(\mathcal{RP}))$ \;\label{line:rapex:add-gcl}
    \If{$s(\mathcal{RP}) = s_{\text{goal}}$}{ \label{line:rapex:goal-test}
        $\mathsf{insert}(\mathcal{RP}, solutions)$
        \label{line:rapex:insert-solution}\tcp*[r]{Alg.~\ref{alg:insert}}
        $\mathbf{continue}$ \;
    }
    \For{$s' \in \mathsf{succ}(s(\mathcal{RP}))$}{
        $\mathcal{RP}' \gets \langle \mathbf{r} + \mathbf{c}_{\mathcal{R}}(\langle s(\mathcal{RP}), s' \rangle), \mathsf{extend}(x, \langle s(\mathcal{RP}), s' \rangle) \rangle$ \;\label{line:rapex:generate-child}
        \If{$\mathsf{is\_dominated}(\mathcal{RP}', G_{\rm{cl}}, solutions)$ \KwThen \tcp*[f]{Alg.~\ref{alg:is_dominated}}}{ \label{line:rapex:prune-child}
            $\mathbf{continue}$ \;
        }
        $\mathsf{insert}(\mathcal{RP}', \open)$ \tcp*[r]{Alg.~\ref{alg:insert}} \label{line:rapex:insert-open}
    }
    $\mathbf{return} \{ x : \langle \mathbf{r}, x \rangle \in solutions \}$\label{line:rapex:return}
}
\caption{$\mathsf{RA^*pex}(S, E, \mathcal{R}, s_{\text{start}}, s_{\text{goal}}, \mathbf{h}, \EPS)$}
\label{alg:RApex}
\end{algorithm}

The baseline version of \rapexAlg, outlined in Alg. \ref{alg:RApex} closely follows the structure of \apexAlg, but replaces vector-based dominance ($\preceq$ and $\preceq_{\EPS}$) with rule-dominance ($\lesssim_{\mathcal{R}}$ and $\lesssim_{\mathcal{R}}^{\EPS}$) to account for hierarchical rule priorities. This difference is reflected in the $\mathsf{is\_dominated}$ function, outlined in Alg. \ref{alg:is_dominated}.

\begin{algorithm}[t]
    \If{$\exists \mathcal{RP}' = \langle \mathbf{r}', x' \rangle \in solutions : \mathbf{f}(x') \lesssim_{R}^{\varepsilon} \mathbf{f}(\mathcal{RP})$ \KwThen \tcp*[f]{Alg.~\ref{alg:lesssim_R}}}{\label{line:isdom:solutions-check}
        remove $\langle \mathbf{r}',x' \rangle$ from $solutions$ \;
        add $\langle \mathsf{comp\_wise\_min}(\mathbf{r}, \mathbf{r}'), x'\rangle$ to $solutions$ \label{line:isdom:update-solution} \;
        $\mathbf{return} \; \texttt{TRUE}$ \;
    }
    \If{$\exists w \in G_{\rm{cl}}(s(\mathcal{RP})) : w \lesssim_{\mathcal{R}} \mathbf{r}$ \KwThen }{\label{line:isdom:local-check}
        $\mathbf{return} \; \texttt{TRUE}$ \;
    }
    $\mathbf{return} \; \texttt{FALSE}$ \;
    \caption{$\mathsf{is\_dominated}(\mathcal{RP} = \langle \mathbf{r}, x \rangle, G_{\rm{cl}}, solutions)$}
    \label{alg:is_dominated}
\end{algorithm}
\begin{algorithm}[t]
    Let $\langle v_1^{\mathcal{R}}, v_2^{\mathcal{R}} \ldots v_M^{\mathcal{R}} \rangle$ be a topological order of $R\setminus \sim$ \;
    $Q \gets \langle v_1^{\mathcal{R}}, v_2^{\mathcal{R}}, \ldots, v_M^{\mathcal{R}} \rangle$ \;
    \While{$Q \neq \emptyset$ }{
        $v = Q.pop()$ \;
        \If{$ \exists r \in [v]$ such that $r(x) > (1+\varepsilon)r(y)$ \KwThen }{
            $\mathbf{return} \, \, \texttt{FALSE}$;
        }
        \ElseIf{$\exists r \in [v]$ such that $r(x) < (1 + \varepsilon)r(y)$}{
            remove all successors of $v$ from $Q$ \;
        }
    }
    $\mathbf{return} \, \, \texttt{TRUE}$ \;
    \caption{Approximate dominance check $x \lesssim_{\mathcal{R}}^{\varepsilon} y$ }
    \label{alg:lesssim_R}
\end{algorithm}
\begin{algorithm}[t]
    \For{$\mathcal{RP}' \in list$}{
        $\mathcal{RP}_{new} \gets \mathsf{merge}(\mathcal{RP}, \mathcal{RP}')$ \;\label{line:insert:merge}
        \If{$\mathcal{RP}_{new}$ is $\EPS$-bounded}{ \label{line:insert:eps-bound-check}
            remove $\mathcal{RP'}$ from $list$ \;
            add $\mathcal{RP}_{new}$ to $list$ \;
            $\mathbf{return}$ \;
        }
    }
    add $\mathcal{RP}$ to $list$ \;
    $\mathbf{return}$ \;
    \caption{$\mathsf{insert}(\mathcal{RP}, list)$}
    \label{alg:insert}
\end{algorithm}

\rapexAlg maintains three main data structures throughout the search:
(i)~the priority queue \open, storing candidate rule-apex-realization pairs and ordered according to rule-dominance $\lesssim_{\mathcal{R}}$ (instead of lexicographic ordering as in \apexAlg), since dominance checks are performed over the full rule-value vectors;
(ii)~the set $solutions$, storing the current representative solutions; and
(iii)~a per-state closed structure $G_{\rm{cl}}(\cdot)$ (initialized in Line~\ref{line:rapex:init-gcl}), storing the $\mathbf{f}$-values of rule-apex-realization pairs that have already been expanded at state~$s$.

At each iteration, \rapexAlg extracts a rule-apex-realization pair $\RP = \langle \mathbf{r},x \rangle$ from \open (Line~\ref{line:rapex:extract}) and immediately applies the pruning test $\mathsf{is\_dominated}$ (Line~\ref{line:rapex:prune}, implemented in Alg.~\ref{alg:is_dominated}). This test consults both $solutions$ and the state-local closed set $G_{\rm{cl}}(s(\RP))$ to decide whether $\RP$ can be discarded; the approximate dominance checks within $\mathsf{is\_dominated}$ rely on Alg.~\ref{alg:lesssim_R}.
If $\RP$ is not pruned, we insert its $\mathbf{f}$-value into $G_{\rm{cl}}(s(\RP))$ (Line~\ref{line:rapex:add-gcl}).
\vspace{-1mm}

If $\RP$ reaches the goal state (Line~\ref{line:rapex:goal-test}), we insert it into $solutions$ (Line~\ref{line:rapex:insert-solution}) using Alg. \ref{alg:insert}.
Similar to \apexAlg, \rapexAlg merges rule-apex-realization pairs by taking the component-wise minimum of their rule apexes (i.e., the merged apex is $\min\{\mathbf{r},\mathbf{r}'\}$) and keeping one of the two representative realizations; this is performed by $\mathsf{merge}$ in Alg.~\ref{alg:insert}, Line~\ref{line:insert:merge}.
Otherwise, we generate successors (Line~\ref{line:rapex:generate-child}); each child $\RP'$ is again filtered by the same pruning test (Line~\ref{line:rapex:prune-child}) and, if it survives, inserted into \open via Alg.~\ref{alg:insert} (Line~\ref{line:rapex:insert-open}).
Upon termination, the algorithm returns the representative realizations stored in $solutions$ (Line~\ref{line:rapex:return}).

\vspace{-1mm}
The core operation underlying all dominance checks is the $\EPS$-rule-dominance test $\lesssim_{\mathcal{R}}^{\varepsilon}$, shown in Alg. \ref{alg:lesssim_R}. Since the rulebook $\mathcal{R}$ forms a partial order on $R\setminus \sim$, it can be represented by a directed acyclic graph $G_{\mathcal{R}} = (V_{\mathcal{R}}, E_{\mathcal{R}})$. Every vertex $v_i^{\mathcal{R}} \in V_{\mathcal{R}}$ represents a set of equivalent rules $[v_i^{\mathcal{R}}] \subseteq R$, such that $r \sim r'$ for all $r,r' \in [v_i^{\mathcal{R}}]$. An edge $v_i^{\mathcal{R}} \rightarrow v_j^{\mathcal{R}}$ means that there are rules $r_i \in [ v_i^{\mathcal{R}}]$ and $r_j \in [v_j^{\mathcal{R}}]$ such that $r_i > r_j$, meaning that $r_i$ has a higher rank that $r_j$, also implying that the rules in $[v_i^{\mathcal{R}}]$ have a higher rank than the rules in $[r_j^{\mathcal{R}}]$. The procedure in Alg. \ref{alg:lesssim_R} follows the rulebook structure by traversing the rules in the topological order $\langle v_1^{\mathcal{R}}, v_2^{\mathcal{R}}, \ldots, v_M^{\mathcal{R}} \rangle$ of their equivalence classes in $V_{\mathcal{R}}$, where $M \leq N$ is the number of equivalence classes. It checks for violations that exceed the allowed $\EPS$ bound and prunes the lower-priority rules once sufficient improvement is detected at a higher level. This hierarchical filtering is what fundamentally distinguishes \rapexAlg from \apexAlg and explains why dominance checks are more complex in the rulebook setting.

\vspace{-5mm}
\subsection{Dimensionality Reduction}
\vspace{-2mm}
Recall that dimensionality reduction is more challenging with rulebooks because of the structure of rule-dominance. In \apexAlg, \open is ordered lexicographically by the $\mathbf{f}$-values. Thus, when an apex-path pair $\AP$ is extracted from \open, lexicographic best-first ordering guarantees that the first component of $\mathbf{f}(\AP)$ cannot improve upon that of any previously expanded pair that ends at the same state or any pair already in the solution set, i.e., $f_1(\AP) \geq f_1(\AP')$ for all $\AP' \in G_{\rm{cl}}(s(\AP)) \cup solutions$. This property enables a simple form of dimensionality reduction, i.e., when checking whether $\mathbf{f}(\AP)$ is dominated by any vector in $G_{\rm{cl}}(s(\AP)) \cup solutions$, the first component can be ignored and dominance checks reduce to comparisons over the remaining components only.
\vspace{-2mm}
\begin{example}
    Given three vectors $\mathbf{u} = [1, 3, 4]$, $\mathbf{v}=[2, 3, 4]$, and $\mathbf{w}=[2, 4, 1]$, consider testing whether $\mathbf{u} \preceq \mathbf{w}$ or $\mathbf{v} \preceq \mathbf{w}$. Since $u_1 \leq w_1$, it suffices to compare the suffixes $[3, 4]$ and $[4, 1]$. Since $[3,4] \not \preceq [4,1]$, we can conclude that $u \not \preceq w$. The same reasoning applies to $\mathbf{v}$, yielding $\mathbf{v} \not\preceq \mathbf{w}$.
    \label{ex:apex-dim-reduction}
\end{example}
\vspace{-2.5mm}

When there are hierarchies among the rules, this simplification no longer applies. Although \rapexAlg also guarantees that the first component of $\mathbf{f}(\RP)$ is not smaller than that of any previously expanded pair that ends at $s(\RP)$ or any pair in $solutions$, this information alone is insufficient to determine rule-dominance.
Unlike component-wise dominance, rule-dominance depends not only on non-inferiority at higher-priority rules, but also on the presence of strict improvements that may compensate for degradations at lower-priority rules. Consequently, lexicographic ordering cannot isolate a single component, and dominance checks must explicitly account for the rule hierarchy and distinguish strict from non-strict improvements, as illustrated in the following example.
\vspace{-1mm}

\begin{example}
Consider the same rulebook as in Example \ref{ex:eps-rule-dominance}, namely $\mathcal{R} = \langle R, \lesssim \rangle$ with $R = \{r_1, r_2, r_3\}$, $r_1 > r_2$ and $r_1 > r_3$.
Let $x, y, z$ be realizations with
$\mathbf{c}_{\mathcal{R}}(x) = [1, 3, 4]$, $\mathbf{c}_{\mathcal{R}}(y) = [2, 3, 4]$ and $\mathbf{c}_{\mathcal{R}}(z) = [2, 4, 1]$. For consistency with Example \ref{ex:apex-dim-reduction}, let $\mathbf{u} = \mathbf{c}_{\mathcal{R}}(x)$, $\mathbf{v} = \mathbf{c}_{\mathcal{R}}(y)$ and $\mathbf{w} = \mathbf{c}_{\mathcal{R}}(z)$.

First, consider whether $\mathbf{u} \lesssim_{\mathcal{R}} \mathbf{w}$.
As in Example \ref{ex:apex-dim-reduction}, we have $u_1 \leq w_1$ and the suffixes satisfy $[3,4] \not\lesssim_{\mathcal{R}} [4,1]$ because $u_3 > w_3$. Under component-wise dominance, this would be sufficient to conclude non-dominance. However, under rule-dominance, the degradation at $r_3$ is offset by a strict improvement at a higher-priority rule $r_1 > r_3$ where $u_1 < w_1$. This strict improvement at a more important rule compensates for the degradation at $r_3$, and therefore $\mathbf{u} \lesssim_{\mathcal{R}} \mathbf{w}$.

Now consider whether $\mathbf{v} \lesssim_{\mathcal{R}} \mathbf{w}$. Again, the suffixes satisfy $[3,4] \not\lesssim_{\mathbf{R}} [4,1]$. However, unlike the previous case, there is no strict improvement at any higher-priority rule since $v_1 = w_1$. Since the degradation at $r_3$ is not compensated by a strict improvement at a more important rule, we conclude that $\mathbf{v} \not \lesssim_{\mathcal{R}} \mathbf{w}$.

Note that in both this example and Example \ref{ex:apex-dim-reduction}, we have $v_1,u_1 \leq w_1$. Under rule-dominance, however, the distinction between a strict improvement ($u_1 < w_1)$ and equality ($v_1 = w_1$) is critical. This illustrates why lexicographic ordering alone is insufficient for dimensionality reduction under rule-dominance.
\end{example}
\vspace{-2mm}

\begin{algorithm}[t]
    \If{$\exists \RP^= \in G_{\rm{cl}}^=(s(\RP)) : f_1(\RP^=) < f_1(\RP)$ \KwThen \tcp*[f]{transfer $G_{\rm{cl}}^=$ to $G_{\rm{cl}}^<$}}{ \label{line:isdomdr:transfer}
        \For{$\RP' \in G_{\rm{cl}}^<(s(\RP))$}{
            \If{$\exists \RP'' \in G_{\rm{cl}}^=(s(\RP)) : Tr(\mathbf{f}(\RP'')) \lesssim_{\mathcal{R}} Tr(\mathbf{f}(\RP'))$}{
                remove $\RP'$ from $G_{\rm{cl}}^<(s(\RP))$ \;
            }
        }
        $G_{\rm{cl}}^<(s(\RP)) \gets G_{\rm{cl}}^<(s(\RP)) \cup G_{\rm{cl}}^=(s(\RP))$ \; \label{line:isdomdr:transfer-add}
        $G_{\rm{cl}}^=(s(\RP)) \gets \emptyset$ \;\label{line:isdomdr:transfer-clear}
    }

    \If{$\exists \RP^= \in G_{\rm{cl}}^=(s(\mathcal{RP})) : Tr(\mathbf{f}(\RP^=)) \lesssim_{\mathcal{R}} Tr(\mathbf{f}(\RP))$ \KwThen \label{line:isdomdr:truncated-check}}{
        $\mathbf{return} \; \texttt{TRUE}$ \;
    }
    \If{$\exists \RP^< \in G_{\rm{cl}}^<(s(\mathcal{RP})) : \alpha(\mathbf{f}(\RP^<)) \lesssim_{\mathcal{R}} \alpha(\mathbf{f}(\RP))$ \KwThen \label{line:isdomdr:residual-check}}{
        $\mathbf{return} \; \texttt{TRUE}$ \;
    }
    \If{$\exists \mathcal{RP}' = \langle \mathbf{r}', x' \rangle \in solutions :$ $\mathbf{f}(x') \lesssim_{R}^{\varepsilon} \mathbf{f}(\mathcal{RP})$ \KwThen \label{line:isdomdr:solutions-check}}{
        remove $\langle \mathbf{r}',x' \rangle$ from $solutions$ \;
        add $\langle \mathsf{comp\_wise\_min}(\mathbf{r}, \mathbf{r}'), x')$ to $solutions$ \label{line:isdomdr:update-solution} \;
        $\mathbf{return} \; \texttt{TRUE}$ \;
    }
    $\mathbf{return} \; \texttt{FALSE}$ \;
    \caption{\parbox{\linewidth}{$\mathsf{is\_dominated_{\rm dr}}(\mathcal{RP} = \langle \mathbf{r}, x \rangle, G_{\rm{cl}}^=, G_{\rm{cl}}^<, solutions)$}}
    \label{alg:is_dominated_dim_reduction}
\end{algorithm}

Supporting dimensionality reduction under rule-dominance thus requires maintaining two closed sets for each state $s$ instead of a single set $G_{\rm{cl}}(s)$ used in the baseline \rapexAlg (see Alg. \ref{alg:RApex}) and in \apexAlg. Specifically, \rapexAlg maintains $G_{\rm{cl}}^<$ and $G_{\rm{cl}}^=$ for local dominance checks at each state. Intuitively, for any rule-apex-realization pair $\RP$,
the set $G_{\rm{cl}}^=(s(\RP))$ stores previously expanded pairs $\RP'$ that end at the same state and have the same $f_1$-value, i.e., $s(\RP') = s(\RP)$ and $f_1(\RP') = f_1(\RP)$, while $G_{\rm{cl}}^<(s(\RP))$ stores those that end at the same state but have strictly smaller $f_1$-value. Formally, whenever $\RP$ is checked for dominance, it holds that $\forall \RP' \in G_{\rm{cl}}^=(s(\RP)), f_1(\RP') = f_1(\RP)$  and $\forall \RP' \in G_{\rm{cl}}^<(s(\RP)) : f_1(\RP') < f_1(\RP)$.
\vspace{-0.5mm}

With this partitioning, we update the dominance test by introducing the $\mathsf{is\_dominated_{\rm{dr}}}$ function presented in Alg.~\ref{alg:is_dominated_dim_reduction} and using it in Alg. \ref{alg:RApex} Lines~\ref{line:rapex:prune} and~\ref{line:rapex:prune-child}.
To formalize the dimensionality reduction, we first compute a topological ordering \cite{Cormen:2022:Introduction} $\langle v_1^{\mathcal{R}}, v_2^{\mathcal{R}}, \ldots, v_M^{\mathcal{R}} \rangle$ of the directed acyclic graph $G_{\mathcal{R}}$. We then obtain an ordering $\langle
r_{\sigma(1)}, r_{\sigma(2)}, \ldots, r_{\sigma(N)} \rangle$ of rules such that $r_{\sigma(i)} \not< r_{\sigma(j)}$ for all $i > j$ by simply expanding each equivalence class $[v_i^{\mathcal{R}}]$ in their topological order. Intuitively, this ordering arranges rules in non-increasing priority. Given a rule-value vector $\mathbf{u} = [u_1, \ldots, u_N]$, we reorder its components according to this permutation to obtain $\mathbf{u}_\sigma = [u_{\sigma(1)}, u_{\sigma(2)}, \ldots, u_{\sigma(N)}]$. \open is subsequently ordered lexicographically with respect to these reordered rule-value vectors, as in \apexAlg. For notational simplicity, we henceforth assume that the rules are indexed in topological order, so that $r_i \not< r_j$ for all $i > j$.


Under this convention, $r_1$ is a highest-priority rule, meaning that there is no rule $r_i \in R$ such that $r_i > r_1$ ($r_i$ may have equal priority as $r_1$).
We then define the \emph{truncated rule set} $Tr(R) := \{r_2, \ldots, r_N\}$, which excludes $r_1$ as in \apexAlg. Additionally, we define the \emph{residual rule set} $\alpha(R) := \{r_i \in Tr(R) \ | \ \neg(r_i < r_1) \}$, which excludes $r_1$ and any rule that is strictly lower in priority than $r_1$. By a slight abuse of notation, we use the same operator $Tr(\cdot)$ and $\alpha(\cdot)$ to denote the corresponding operations on rule-value vectors. Specifically, for any $\mathbf{u} \in \mathbb{R}^N$, we define $Tr(\mathbf{u})$ and $\alpha(\mathbf{u})$ as the subvector of $\mathbf{u}$ obtained by keeping only the components corresponding to the rules in $Tr(R)$ and $\alpha(R)$, respectively, with the original order preserved.
Alg.~\ref{alg:is_dominated_dim_reduction} uses $Tr(\cdot)$ and $\alpha(\cdot)$ to perform dimensionality reduction-based rule-dominance.
Conceptually, the algorithm maintains two state-local closed sets: $G_{\rm{cl}}^{=}(s)$ for expanded pairs whose extracted $f_1$-value equals the current best-known value at $s$, and $G_{\rm{cl}}^{<}(s)$ for expanded pairs with strictly smaller $f_1$.
Whenever the extracted $f_1$ at $s$ increases (Line~\ref{line:isdomdr:transfer}), all pairs in $G_{\rm{cl}}^{=}(s)$ become ``strictly better in $r_1$'' than the new pair and are therefore moved into $G_{\rm{cl}}^{<}(s)$ (Lines~\ref{line:isdomdr:transfer-add}--\ref{line:isdomdr:transfer-clear}), after filtering out entries that are already dominated in the residual space $\alpha(\cdot)$.

\vspace{-0.5mm}

After this bookkeeping, the dominance test proceeds in three stages.
First, it checks for a witness in $G_{\rm{cl}}^{=}(s(\RP))$ using only the truncated vector $Tr(\mathbf{f}(\cdot))$ (Line~\ref{line:isdomdr:truncated-check}): since all candidates have the same $f_1$-value, rule-dominance cannot rely on a strict improvement in $r_1$, and we can safely ignore $r_1$ and compare the remaining rules.
Second, it checks $G_{\rm{cl}}^{<}(s(\RP))$ in the residual space $\alpha(\cdot)$ (Line~\ref{line:isdomdr:residual-check}): here every candidate already has a strict improvement in $r_1$, so any rule that is strictly lower priority than $r_1$ is automatically ``covered'' and can be removed from the comparison.
Finally, it performs the global approximate-dominance check against $solutions$ exactly as in the baseline algorithm (Line~\ref{line:isdomdr:solutions-check}).


\section{Theoretical Analysis}
\vspace{-3mm}
This section provides a formal guarantee that \rapexAlg solves  Problem \ref{prob:rapprox}. In summary, the correctness of \rapexAlg hinges on a key structural property of $\EPS$-rule-dominance $\lesssim_{\mathcal{R}}^{\EPS}$, namely its one-sided transitivity with respect to exact rule-dominance $\lesssim_{\mathcal{R}}$.

\vspace{-1mm}
\begin{definition}
Let $\precsim$ and $\precsim'$ be two binary relations on a set $X$.
We say that $\precsim'$ is \emph{one-sided transitive with respect to} $\precsim$ if, for any $x,y,z \in X$,
$x \precsim' y$ and $y \precsim z$ implies $x \precsim' z$.
\end{definition}
\vspace{-1mm}

With the definition of $\EPS$-rule-dominance given in Definition \ref{def:eps-rule-dom}, it can be shown that $\lesssim_{\mathcal{R}}^{\EPS}$ is one-sided transitive with respect to $\lesssim_{\mathcal{R}}$ (see Lemma \ref{lem:transitivity} below). This property is not guaranteed for alternative notions of $\EPS$-rule-dominance as demonstrated in the following example.

\vspace{-2mm}
\begin{example}
Consider an alternative definition of $\lesssim_{\mathcal{R}}^{\EPS}$:
Given realizations $x, y \in \Sigma$ and a vector $\EPS \geq \boldsymbol 0$, define $x \lesssim_{\mathcal{R}}^{\EPS} y$ if for any rule $r_j \in R$ such that
$r_j(x) > (1+\varepsilon_{j})r_j(y)$, there exists a higher-priority rule $r_i > r_j$ such that $r_i(x) < \frac{r_i(y)}{1+\varepsilon_{i}}$.
Intuitively, any degradation at a rule beyond the allowed tolerance must be compensated by a sufficiently strong improvement at a higher-priority rule.
The motivation behind this definition is to treat small relative differences at higher-priority rules as negligible under approximation. In particular, if two realizations differ only slightly at a high-priority rule, then differences at lower-priority rules may still be considered meaningful.
For example, consider rules $R=\{r_1,r_2\}$ with $r_1 > r_2$ and rule-value vectors $\mathbf{u} = [1, 10^6]$ and $\mathbf{v} = [1.001, 1]$. With $\EPS=[0.1,0.1]$, the relative difference at the higher-priority rule $r_1$ is within tolerance, while $\mathbf{v}$ is significantly better at the lower-priority rule $r_2$. Under this alternative definition, we obtain $\mathbf{v} \lesssim_{\mathcal{R}}^{\EPS} \mathbf{u}$ but $\mathbf{u} \not\lesssim_{\mathcal{R}}^{\EPS} \mathbf{v}$, i.e., $\mathbf{v}$ is strictly preferred to $\mathbf{u}$. In contrast, under Definition \ref{def:eps-rule-dom}, we have both $\mathbf{v} \lesssim_{\mathcal{R}}^{\EPS} \mathbf{u}$ and $\mathbf{u} \lesssim_{\mathcal{R}}^{\EPS} \mathbf{v}$, so $\mathbf{u}$ and $\mathbf{v}$ are treated as equivalent.

Despite this intuition, this alternative definition does not satisfy one-sided transitivity. Consider rules $R=\{r_1,r_2,r_3\}$ with $r_1 > r_2$ and $r_1 > r_3$, and rule-value vectors $\mathbf{u} = [2,2,2], \mathbf{v} = [4,1,1], \mathbf{w} = [3,4,2]$ with $\EPS = [1,1,1]$.
It is easy to check that $\mathbf{w} \lesssim_{\mathcal{R}}^{\EPS} \mathbf{u}$ (since there is no rule $r_j$ such that $w_j > 2u_j$) and $\mathbf{u} \lesssim_{\mathcal{R}} \mathbf{v}$ (since $u_1 < v_1$). However, based on the alternative definition, $\mathbf{w} \not\lesssim_{\mathcal{R}}^{\EPS} \mathbf{v}$ because $w_1 = 3 \not< 0.5 v_1 = 2$.
\end{example}
\vspace{-2mm}

One-sided transitivity allows us to reason about sets of realizations with $\varepsilon$-bounded rule-apex-realization pairs and ensures that dominance relations are safely propagated across search expansions. The following lemmas and the theorem, with $\EPS$-rule-dominance defined in Definition \ref{def:eps-rule-dom}, show that \rapexAlg correctly solves Problem \ref{prob:rapprox}.

\vspace{1em}
We begin by establishing a fundamental property of the search order.

\setcounter{lemma}{0}
\begin{lemma}\label{lem:monotonic_non_decreasing}
    The sequence of extracted rule-apex-realization pairs has monotonically non-decreasing $f_1$-values.
\end{lemma}
\begin{proof}
The rule-apex-realization pair extracted from \open has the smallest $f_1$-value due to the lexicographic ordering. Also, any newly generated pair has an $f_1$-value that is greater than or equal to that of its parent pair, and the resulting pair of a merge can not have an $f_1$-value smaller than both of its merged pairs (their rule apexes). Thus, the sequence of extracted pairs has monotonically non-decreasing $f_1$-values.
\end{proof}

The next two lemmas justify the correctness of pruning based on previously expanded nodes.

\begin{lemma} \label{lem:truncated}
    If there exists a rule-apex-realization pair $\RP'$ in $G_{\rm{cl}}^{=}(s(\mathcal{RP}))$ the truncated $\mathbf{f}$-value of which weakly rule-dominates the truncated $\mathbf{f}$-value of some realization-pair $\mathcal{RP}$ on Line \ref{line:isdomdr:truncated-check} of the Algorithm \ref{alg:is_dominated_dim_reduction}, then the $\mathbf{f}$-value of $\RP'$ weakly rule-dominates the $\mathbf{f}$-value of $\mathcal{RP}$.
\end{lemma}
\begin{proof}
The rule-apex-realization pair $\RP'$ has been expanded and added to $G_{\rm{cl}}^=(s(\RP))$ before $\RP$ was extracted from \open. Let $$\RP_0=\RP', \RP_1, \RP_2, \ldots \RP_k=\RP$$ be the sequence of rule-apex-realization pairs extracted since then. Because $\RP'$ is still in $G_{\rm{cl}}^=(s(\RP))$, the Line \ref{line:isdomdr:transfer} of the Algorithm \ref{alg:is_dominated_dim_reduction} never held. According to Lemma \ref{lem:monotonic_non_decreasing} and since the heuristic function is consistent, it follows that $f_1(\RP_i) = f_1(\RP_{i+1})$ for all $i < k$. Thus, $f_1(\RP') = f_1(\RP)$ which along with the premise in the lemma implies $\mathbf{f}(\RP') \lesssim_{\mathcal{R}} \mathbf{f}(\RP)$.
\end{proof}

\begin{lemma} \label{lem:residual}
If there exists a rule-apex-realization pair $\RP'$ in $G_{\rm{cl}}^< (s(\mathcal{RP}))$
the residual $\mathbf{f}$-value of which weakly rule-dominates the residual $\mathbf{f}$-value of some rule-apex-realization pair $\mathcal{RP}$ on Line \ref{line:isdomdr:residual-check} of the Algorithm \ref{alg:is_dominated_dim_reduction}, then it also holds that the $\mathbf{f}$-value of $\RP'$ weakly rule-dominates the $\mathbf{f}$-value of $\mathcal{RP}$.
\end{lemma}
\begin{proof}
    Since $\RP'$ was previously expanded and added to $G_{\rm{cl}}^<(s(\RP))$, the condition on Line \ref{line:isdomdr:transfer} of the Algorithm \ref{alg:is_dominated_dim_reduction} held at least once. According to Lemma \ref{lem:monotonic_non_decreasing} and because the heuristic function is consistent, it holds that $f_1(\mathcal{RP')} < f_1(\mathcal{RP})$. Thus, the $\mathbf{f}$-value of $\mathcal{RP'}$ weakly rule-dominates the one of $\mathcal{RP}$.
\end{proof}

We then turn to the correctness of the solution-set maintenance.
\begin{lemma}
    Let $\RP = \langle \mathbf{r}, x \rangle$ be a rule-apex-realization pair and $y$ be any realization such that $\mathbf{r} \lesssim_{\mathcal{R}} \mathbf{c}_{\mathcal{R}}(y)$. If we merge $\RP$ with another rule-apex-realization pair then the rule apex of the resulting rule-apex-realization pair will still weakly rule-dominate ($\lesssim_{\mathcal{R}}$) $y$.
    \label{lem:merge_consistency}
\end{lemma}
\begin{proof}
    The rule apex of the resulting rule-apex-realization pair is the component-wise minimum of the two merged rule apexes.
\end{proof}

The next lemma shows that $\EPS$-rule-dominance satisfies one-sided transitivity with respect to exact rule-dominance. In other words, approximate dominance is preserved when composed with exact dominance. This property allows $\rapexAlg$ to reason about sets of realizations using rule-apex-realization pairs and is crucial for establishing the algorithm’s approximation guarantees. Specifically, for a given $\varepsilon$-bounded rule-apex-realization pair $\RP=\langle \mathbf{r}, x\rangle$, it allows the representative realization $x$ to $\EPS$-rule-dominate all the realizations that the pair represents.

\begin{lemma}
Let $x, y, z$ be realizations such that $x \lesssim_{\mathcal{R}}^{\EPS} y$ and $y \lesssim_{\mathcal{R}} z$, then $x \lesssim_{\mathcal{R}}^{\EPS}z$.
\label{lem:transitivity}
\end{lemma}
\begin{proof}
Let $\mathcal{R} = \langle R, \leq \rangle$ be a rulebook with $R = \{r_1, \ldots, r_N\}$. Assume $x \lesssim_{\mathcal{R}}^{\EPS} y$ and $y \lesssim_{\mathcal{R}} z$. To show that $x \lesssim_{\mathbb{R}}^{\EPS} z$, suppose there is a rule $r_i \in R$ such that $r_i(x) > (1 + \varepsilon_{i})r_i(z)$. We need to show that there exists a more important rule $r_j > r_i$ such that $r_j(x) < (1 + \varepsilon_{j}) r_j(z)$. Consider two possibilities.
\begin{itemize}
    \item $r_i(x) \not = (1 + \varepsilon_{i}) r_i(y)$. If $r_i(x) > (1 + \varepsilon_{i}) r_i(y)$, then $x \lesssim_{\mathcal{R}}^{\varepsilon} y$ guarantees that there exists a rule $r_k > r_i$ such that $r_k(x) < (1 + \varepsilon_{k}) r_k(y)$. If instead $r_i(x) < (1 + \varepsilon_{i}) r_i(y)$, then we can set $r_k = r_i$. Thus, in either case, there exists $r_k \geq r_i$ such that $r_k(x) < (1 + \varepsilon_{k})r_k(y)$. We are done if $r_k(y) \leq r_k(z)$, since we can take $r_j = r_k$. Otherwise, $r_k(y) > r_k(z)$ and $y \lesssim_{\mathcal{R}} z$ imply that there  exists $r_l > r_k$ such that $r_l(y) < r_l(z)$. Again, we are done if $r_l(x) \leq (1 + \varepsilon_{l})r_l(y)$, and if not, there has to be some rule $r_m > r_l$ such that $r_m(x) < (1 + \varepsilon_{m}) r_m(y)$. Repeating this argument construncts an increasing chain $r_i \leq r_k < r_l < r_m \dots $. Since $R$ is finite, the chain has to stop, which is only possible if there exists $n \in \{1, \ldots, N\}$ such that $r_n(x) < (1 + \varepsilon_{n})r_n(z)$. \vspace{0.5em}
    \item $r_i(x) = (1 + \varepsilon_{i})r_i(y) $. Then $r_i(x) > (1 + \varepsilon_{i})r_i(z)$ implies that $r_i(y) > r_i(z)$. Because $y \lesssim_{\mathcal{R}}z$, there must exist a rule $r_k > r_i$ such that $r_k(y) < r_k(z)$. We are done if $r_k(x) \leq (1 + \varepsilon_{k})r_k(y)$, since one can take $r_j = r_k$. Otherwise, there must exist a rule $r_l > r_k$ such that $r_l(x) < (1 + \varepsilon_{l})r_l(y)$. Applying the same reasoning as in the previous case, we again obtain an increasing chain $r_i \leq r_k < r_l \ldots$ which must terminate at some $r_n$ with $r_n(x) < (1 + \varepsilon_{n})r_n(z)$.
\end{itemize}
\end{proof}

\begin{lemma}
    For any prefix $x_l = [s_1, s_2 \ldots s_l ]$ of any solution $$x = [ s_1(=s_{\text{start}}), s_2 \ldots s_{L}(=s_{\text{goal}})]$$ with $1 \leq l \leq L$, when the algorithm terminates, there exists either (\emph{Case 1:}) an expanded rule-apex-realization pair $\mathcal{RP}$ (that is, one that reaches Line \ref{line:rapex:add-gcl} of the Alg \ref{alg:RApex}) that ends at the state $s_l$ and whose rule apex weakly rule-dominates the $\mathbf{g}$-value of the realization $x_l$ or (\emph{Case 2:}) a rule-apex-realization pair $\mathcal{RP}$ in the solution set such that the $\mathbf{f}$-value of its representative realization $\bm{\varepsilon}$-rule-dominates the $\mathbf{f}$-value of the realization $x_l$.
    \label{lem:induction}
\end{lemma}
\vspace{-1mm}
\begin{proof}
The proof is done via induction. The lemma holds for $l = 1$ and any solution since the realization pair $\mathcal{RP}=\langle \boldsymbol{0}, [s_{\text{start}}]\rangle$ gets expanded and has the properties required for Case 1. Now assume that the lemma holds for some $l < L$ and any solution. Then, we prove that it is also true for $l+1$ and this solution.

Assume that Case 1 holds for $l$ and consider both $\mathcal{RP}$ - the realization-pair mentioned there and its potential child realization pair $\mathcal{RP}'$ created on Line \ref{line:rapex:generate-child} of the Algorithm \ref{alg:RApex} for $s' = s_{l+1}$. The realization pair $\mathcal{RP}'$ ends at the state $s_{l+1}$ and its rule apex weakly rule-dominates the $\mathbf{g}$-value of the realization $x_{l+1}$, which implies that its $\mathbf{f}$-value weakly rule-dominates the $\mathbf{f}$-value of the realization $x_{l+1}$. We distinguish three cases:
\begin{enumerate}
    \item First, the condition on Line \ref{line:isdomdr:solutions-check} in Alg \ref{alg:is_dominated_dim_reduction} (or Line \ref{line:isdom:solutions-check} in the Alg \ref{alg:is_dominated}) holds for some realization pair in the solution set, meaning, the $\mathbf{f}$-value of the representative path of this realization pair $\EPS$-rule-dominates the $\mathbf{f}$-value of the realization pair $\mathcal{RP'}$. The algorithm replaces this realization pair with a new realization-pair $\mathcal{RP''}$ in the solution set on Line \ref{line:isdomdr:update-solution} in the Alg \ref{alg:is_dominated_dim_reduction} (or Line \ref{line:isdom:update-solution} in the Alg \ref{alg:is_dominated}). The pair $\mathcal{RP}''$ stays in the solution set but the algorithm might merge it several times with other rule-apex-realization pairs on Line \ref{line:insert:merge} of Algorithm \ref{alg:insert} before it terminates. The rule apex of the pair $\mathcal{RP}''$ weakly rule-dominates the $\mathbf{f}$-value of the realization $x_{l+1}$ (since this rule apex is the component-wise minimum of the $\mathbf{f}$-value of the pair $\mathcal{RP}'$ and another rule apex, and hence weakly rule-dominates the $\mathbf{f}$-value of the pair $\RP'$, which in turn weakly rule-dominates the $\mathbf{f}$-value of the realization $x_{l+1}$) and merging it with other realization-pairs does not change this property according to Lemma \ref{lem:merge_consistency}. Since the pair $\RP''$ also remains $\EPS$-bounded, the $\mathbf{f}$-value of its representative path always $\EPS$-rule-dominates the $\mathbf{f}$-value of itself, which equals its rule-apex. Put together, according to Lemma \ref{lem:transitivity} the $\mathbf{f}$-value of its representative path $\EPS$-rule-dominates the $\mathbf{f}$-value of the realization $x_{l+1}$. Thus, the merged realization pair satisfies Case 2 for $l+1$. \vspace{0.5em}
    \item Second, the condition on Line \ref{line:isdomdr:truncated-check} or Line \ref{line:isdomdr:residual-check} of the Alg \ref{alg:is_dominated_dim_reduction} (or Line \ref{line:isdom:local-check} in the Alg \ref{alg:is_dominated}) holds , meaning, there exists a rule-apex-realization realization pair $\RP^=$ in $G_{\rm{cl}}^=(s(\mathcal{RP}'))$ the truncated $\mathbf{f}$-value of which weakly rule-dominates the truncated $\mathbf{f}$-value of the rule-apex-realization pair $\mathcal{RP'}$ or there exists a rule-apex-realization pair $\RP^<$ in $G_{rm{cl}}^<(s(\RP'))$ the residual $\mathbf{f}$-value of which weakly rule-dominates that of the pair $\RP'$. In either case, the expanded pairs $\RP^=$ and $\RP^<$ ending at the state $s_{l+1}$ have $\mathbf{f}$-values that weakly rule-dominate the $\mathbf{f}$-value of the pair $\RP'$ according to Lemmas \ref{lem:truncated} and \ref{lem:residual} respectively. Thus, their rule apex weakly rule-dominate the rule apex of the pair $\mathcal{RP'}$. Thus, either $\mathcal{RP}^=$ or $\RP^<$ satisfies Case 1 for $l+1$ since the rule apex of the pair $\mathcal{RP'}$ in turn weakly rule-dominates the $\mathbf{g}$-value of the realization $x_{l+1}$.
    \item Otherwise, the algorithm executes Line \ref{line:rapex:insert-open} of the Alg \ref{alg:RApex} for the rule-apex-realization pair $\mathcal{RP'}$, where the pair is inserted into $\open$, perhaps after having been merged with another realization-pair in the Algorithm \ref{alg:insert}. The algorithm might merge it several more times before finally extracting it. Its rule apex weakly rule-dominates the $\mathbf{g}$-value of the realization $x_{l+1}$ and merging it with other realization pairs does not change this property according to Lemma \ref{lem:merge_consistency}. Thus, if this realization pair is expanded, it satisfies Case 1 for $l+1$. If it is extracted but not expanded, then the pruning conditions in the Alg \ref{alg:is_dominated_dim_reduction} (or the Alg \ref{alg:is_dominated}) hold, and as we have already proved, Case 1 or Case 2 holds.
\end{enumerate}

Assume that Case 2 holds for $l$ and consider the rule-apex-realization pair mentioned there. The $\mathbf{f}$-value of the representative realization of this pair $\EPS$-rule-dominates the $\mathbf{f}$-value of the realization $x_{l}$. Since the heuristic function is consistent, the $\mathbf{f}$-value of the realization $x_l$ in turn weakly rule-dominates the $\mathbf{f}$-value of the realization $x_{l+1}$. Thus, according to Lemma \ref{lem:transitivity}, this realization-pair satisfies Case 2 for $l+1$.
\end{proof}

Finally, we combine these results to prove Theorem \ref{thm:RA*pex-correctness}.

\begin{theorem}
Upon termination of \rapexAlg, for every solution $y \in \Sigma_{\text{sol}}$, there exists a rule-apex-realization $\RP = \langle \mathbf{r},x \rangle \in solution$ whose representative path $\EPS$-rule-dominates $y$, i.e, $x \lesssim_{\mathcal{R}}^{\EPS} y$.
\label{thm:RA*pex-correctness}
\end{theorem}
\begin{proof}
    Lemma \ref{lem:induction} holds for prefix $x_L = x$ of any solution $x$. In case its Case 2 holds, the theorem holds by definition for realization $x$ since the $\mathbf{f}$-values of the solutions are equal to their costs. In case its Case 1 holds, consider the rule-apex-realization pair as mentioned. This pair ends at the the goal state, and the algorithm thus executed the Line \ref{line:rapex:insert-solution} in Alg \ref{alg:RApex} for it, where the pair was inserted into the solution set, maybe after having been merged with another rule-apex-realization pair on Line \ref{line:insert:merge} of the Algorithm \ref{alg:insert}. The pair stays in the solution set, and might be merged several more times with other rule-apex-realization pairs before the algorithm terminates. The rule apex of the pair weakly rule-dominates the $\mathbf{g}$-value of the path $x$ according to Lemma \ref{lem:induction} and merging it with other pairs does not change this fact according to Lemma \ref{lem:merge_consistency}. Since the pair also remains $\EPS$-bounded (due to the conditions on Line \ref{line:insert:eps-bound-check} in Alg \ref{alg:insert}, Line \ref{line:isdomdr:solutions-check} in Alg \ref{alg:is_dominated_dim_reduction}, Lemma \ref{lem:monotonic_non_decreasing}, and the consistency of the heuristic function), the $\mathbf{f}$-value of its representative realization always $\EPS$-rule-dominates the $\mathbf{f}$-value of itself, which equals its rule apex. Putting it all together, the $\mathbf{f}$-value of its representative realization $\EPS$-rule-dominates the $\mathbf{g}$-value of the realization $x$ according to Lemma \ref{lem:transitivity}. Thus, the theorem holds by definition for the realization $x$ since the $\mathbf{g}$- and $\mathbf{f}$-values of the solutions are equal to their costs.
\end{proof}

\vspace{-6mm}
\section{Experimental Results}
\vspace{-2mm}
We evaluate the efficiency of \rapexAlg across three different problem settings that differ in both the presence of an objective hierarchy and the use of approximation.
In setting $\textbf{(S1)}$ objectives are non-hierarchical
while
in setting $\textbf{(S2)}$ and $\textbf{(S3)}$ objectives are hierarchical and we do not / do allow an approximation, respectively.
Consequentially, we compare \rapexAlg against a state-of-the-art algorithm appropriate for each setting:
in $\textbf{(S1)}$ against~\apexAlg,
in $\textbf{(S2)}$ against rulebook-based complete control synthesis \cite{WongpiromsarnSF26}, which we refer to as \rbexAlg
and
in $\textbf{(S3)}$ against an approximate extension of \rbexAlg, which we refer to as \rbapAlg.
The experiments were conducted on the BAY roadmap from the 9th DIMACS Implementation Challenge: Shortest Path \cite{DIMACSImplementationChallenge2005}, as well as on randomly generated graphs. All experiments were run on the Macbook Air with an Apple M4 Chip, 24GB of memory, and a five-minute runtime limit per instance. All algorithms were implemented in \verb!C++!, reusing common code from \apexAlg and \rbexAlg wherever possible.\footnote{https://github.com/Infus3d/Rulebook\_approximation}

We used the same 25 generated 3-objective roadmap instances that Zhang et al.~\cite{zhang2022pex} used with randomly selected start and goal states for each roadmap.
A fourth objective was added by assigning to each edge a value randomly drawn from $\{0,1\}$. The heuristic function $\mathbf{h}$ was set as a vector, where each component corresponds to a single objective and is given by the minimum achievable cost from a state $s$ to the goal with respect to that objective, computed using Dijkstra's algorithm. Since computing $\mathbf{h}$ accounted for only a small fraction of the total runtime, all reported runtimes exclude this preprocessing step.

While \rapexAlg outperforms the existing methods in both settings $\textbf{(S2)}$ and $\textbf{(S3)}$, we observed that it performs the best in the setting $\textbf{(S3)}$ since the approximated planning takes advantage of the grouping of the similar cost paths in the search frontier, making it up to two orders of magnitude faster than the exact planning in $\textbf{(S2)}$ on the 4-objective roadmap instances from the BAY dataset.

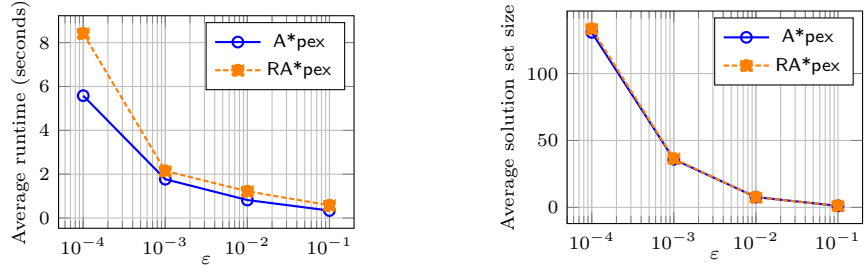
\begin{figure}[t]
\centering
\begin{minipage}{0.45\linewidth}
\centering
\begin{tikzpicture}
\begin{axis}[
    width=1.0\linewidth,  
    height=0.8\linewidth, 
    xlabel={$\varepsilon$},
    ylabel={Average runtime (seconds)},
    xlabel style={yshift=2mm}, 
    ylabel style={yshift=-1mm}, 
    xmode=log,
    grid=both,
    legend pos=north east,
    tick label style={font=\scriptsize},  
    label style={font=\scriptsize},       
    legend style={font=\scriptsize}       
]
\addplot[
    color=blue,
    mark=o,
    thick,
    solid
] coordinates {
    (0.0001, 5.58)
    (0.001, 1.77)
    (0.01, 0.82)
    (0.1, 0.34)
};
\addlegendentry{\apexAlg}
\addplot[
    color=orange,
    mark=square*,
    thick,
    dash pattern=on 2pt off 1pt
] coordinates {
    (0.0001, 8.42)
    (0.001, 2.15)
    (0.01, 1.22)
    (0.1, 0.58)
};
\addlegendentry{\rapexAlg}
\end{axis}
\end{tikzpicture}
\end{minipage}
\hfill
\begin{minipage}{0.45\linewidth}
\centering
\begin{tikzpicture}
\begin{axis}[
    width=1.0\linewidth,  
    height=0.8\linewidth, 
    xlabel={$\varepsilon$},
    ylabel={Average solution set size},
    xlabel style={yshift=2mm}, 
    ylabel style={yshift=-1mm}, 
    xmode=log,
    grid=both,
    legend pos=north east,
    tick label style={font=\scriptsize},  
    label style={font=\scriptsize},       
    legend style={font=\scriptsize}       
]
\addplot[
    color=blue,
    mark=o,
    thick,
    solid
] coordinates {
    (0.0001, 130.76)
    (0.001, 35.92)
    (0.01, 7.68)
    (0.1, 1.12)
};
\addlegendentry{\apexAlg}
\addplot[
    color=orange,
    mark=square*,
    thick,
    dash pattern=on 2pt off 1pt
] coordinates {
    (0.0001, 133.48)
    (0.001, 36.4)
    (0.01, 7.56)
    (0.1, 1.2)
};
\addlegendentry{\rapexAlg}
\end{axis}
\end{tikzpicture}
\end{minipage}
\vspace{-4mm}
\caption{Comparison of average runtimes (left) and average solution set sizes (right) on 3-objective roadmap instances with 321,270 states and 794,830 edges, for different approximation factors $\varepsilon$.}
\vspace{-1mm}
\label{fig:bay_comparison}
\end{figure}

\vspace{-4.5mm}
\subsection*{Setting (S1): Approximate Planning without Objective Hierarchy}
\vspace{-1.5mm}
Fig. \ref{fig:bay_comparison} shows the average runtimes (in seconds) of \apexAlg and \rapexAlg over all roadmap instances in the BAY dataset, plotted as a function of the approximation factor.
In these experiments, $\EPS$ is a vector whose components are set to the same scalar value $\varepsilon$, which is now represented on the x-axis for clarity..

Across all instances, the runtime of \rapexAlg is within $1.5\times$ the runtime of \apexAlg while producing approximate solution sets of comparable size. This shows that, in the absence of objective hierarchies, \rapexAlg incurs only modest overhead relative to \apexAlg while maintaining similar solution quality.

\vspace{-4.5mm}
\subsection*{Setting (S2): Exact Planning with Hierarchical Objectives}
\vspace{-1.5mm}
%
Fig. \ref{fig:pareto_optimal} (Left) compares \rapexAlg and \rbexAlg on 3-objective random graph instances of increasing size, reporting both runtime (in seconds) and the size of the returned solution sets. The rule hierarchy is $r_1 > r_2$ and $r_1 > r_3$. On the smallest instances, both algorithms exhibit comparable runtimes. As graph size increases, however, the runtime of \rbexAlg grows rapidly, becoming more than two orders of magnitude slower than \rapexAlg on the largest instances.

Fig. \ref{fig:pareto_optimal} (Right) shows detailed results for 4-objective roadmap instances from the BAY dataset, with rule hierarchy $r_4 > r_1$, $r_4 > r_2$, and $r_4 > r_3$. In all of the instances, \rbexAlg failed to finish within five minutes, while \rapexAlg completed all but the three largest instances. For instances with up to one million explored nodes, \rapexAlg found the solution set in about a minute, while for instances with up to 100,000 explored nodes, it completed in under 5 seconds.


\begin{figure}[t]
\centering
\begin{minipage}{0.43\linewidth}
\centering
\begin{tikzpicture}
\begin{axis}[
    width=1.0\linewidth,
    height=1.0\linewidth,
    xlabel={Graph size},
    ylabel={Runtime (seconds)},
    xlabel style={yshift=2mm}, 
    ylabel style={yshift=-1mm}, 
    ymode=log,
    grid=both,
    legend style={font=\scriptsize, at={(1.02,0.01)}, anchor=south east},
    tick label style={font=\scriptsize},
    label style={font=\scriptsize},
]
\addplot[
    color=blue,
    mark=o,
    thick
] coordinates {
    (25, 0.0009)
    (100, 0.14)
    (300, 4.1)
    (500, 23.4)
    (700, 93.5)
};
\addlegendentry{\rbexAlg runtime}
\addplot[
    color=orange,
    mark=square*,
    thick,
    dash pattern=on 2pt off 1pt
] coordinates {
    (25, 0.0002)
    (100, 0.003)
    (300, 0.011)
    (500, 0.07)
    (700, 0.11)
};
\addlegendentry{\rapexAlg runtime}
\addlegendimage{
    color=gray,
    mark=triangle*,
    thick
}
\addlegendentry{Solution set size}
\end{axis}
\begin{axis}[
    width=1.0\linewidth,
    height=1.0\linewidth,
    xlabel style={yshift=2mm}, 
    ylabel style={yshift=1mm}, 
    axis x line=none,
    axis y line*=right,
    ylabel={Solution set size},
    tick label style={font=\scriptsize},
    label style={font=\scriptsize},
]
\addplot[
    color=gray,
    mark=triangle*,
    thick
] coordinates {
    (25, 3)
    (100, 4)
    (300, 7)
    (500, 8)
    (700, 9)
};
\end{axis}
\end{tikzpicture}
\end{minipage}
\hfill
\begin{minipage}{0.43\linewidth}
\centering
\begin{tikzpicture}
\begin{axis}[
    width=1.0\linewidth,
    height=1.0\linewidth,
    xlabel={Instance},
    ylabel={Number of nodes ($\times 10^6$)},
    scaled y ticks=false,
    ytick={0,500000,1000000,1500000},
    yticklabels={0,0.5,1.0,1.5},
    xlabel style={yshift=2mm}, 
    ylabel style={yshift=-1mm}, 
    grid=both,
    legend style={font=\scriptsize, at={(1.12,0.01)}, anchor=south east},
    tick label style={font=\scriptsize},
    label style={font=\scriptsize},
]
\addplot[
    color=blue,
    mark=o,
    thick
] coordinates {
    (1, 26939)
    (2, 54259)
    (3, 76575)
    (4, 413077)
    (5, 829329)
    (6, 830249)
    (7, 1086092)
    (8, 1362012)
    (9, 1525848)
};
\addlegendentry{Nodes explored}

\addplot[
    color=gray,
    mark=triangle*,
    thick,
] coordinates {
    (1, 23389)
    (2, 47496)
    (3, 66016)
    (4, 363586)
    (5, 742390)
    (6, 714992)
    (7, 938519)
    (8, 1175205)
    (9, 1287778)
};
\addlegendentry{Nodes expanded}

\addlegendimage{
    color=orange,
    mark=square*,
    thick,
    dash pattern=on 2pt off 1pt
}
\addlegendentry{\rapexAlg runtime}
\end{axis}

\begin{axis}[
    width=1.0\linewidth,
    height=1.0\linewidth,
    xlabel style={yshift=2mm}, 
    ylabel style={yshift=1mm}, 
    axis x line=none,
    axis y line*=right,
    ylabel={Runtime (seconds)},
    ymode=log,
    tick label style={font=\scriptsize},
    label style={font=\scriptsize},
]

\addplot[
    color=orange,
    mark=square*,
    thick,
    dash pattern=on 2pt off 1pt
] coordinates {
    (1, 3.99)
    (2, 3.43)
    (3, 2.6)
    (4, 17.6)
    (5, 60.6)
    (6, 65.2982)
    (7, 69.1)
    (8, 140.49)
    (9, 150.73)
};

\end{axis}
\end{tikzpicture}
\end{minipage}
\vspace{-4mm}
\caption{(Left) Runtimes (in seconds) and sizes of solution sets for \rbexAlg and \rapexAlg on random graphs with 3 rules ($r_1 > r_2, r_3$) under exact dominance (i.e., for $\EPS = \mathbf{0}$). (Right) Runtimes (in seconds) for \rapexAlg on roadmap instances with 4 rules ($r_4 > r_1, r_2, r_3$), along with the number of nodes explored and expanded during the search.
}
\label{fig:pareto_optimal}
\end{figure}
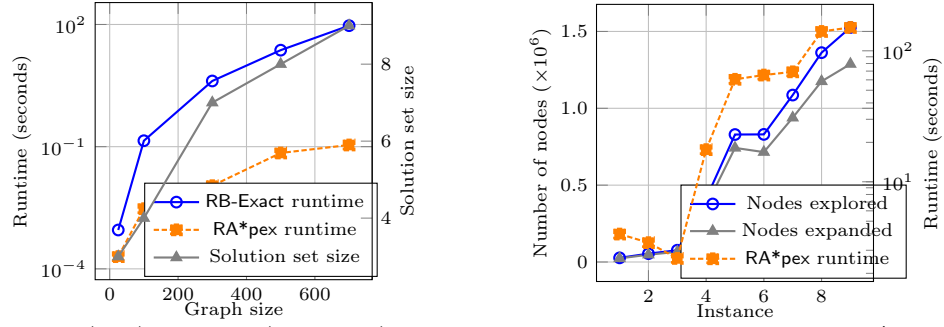

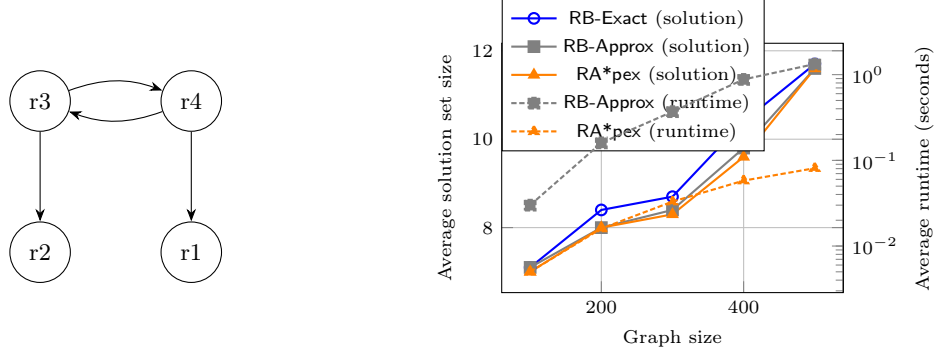
\begin{figure}[t]
\centering
\begin{minipage}{0.3\linewidth}
\centering
\begin{tikzpicture}[
    ->, >=Stealth, 
    node distance=2cm,
    every node/.style={draw, circle, minimum size=8mm, font=\small}
]
\node (r3) {r3};
\node (r4) [right of = r3] {r4};
\node (r1) [below of = r4] {r1};
\node (r2) [below of = r3] {r2};
\draw[bend left=20] (r3) to (r4);
\draw[bend left=20] (r4) to (r3);
\draw (r4) -- (r1);
\draw (r3) -- (r2);
\end{tikzpicture}
\end{minipage}
\hfill
\begin{minipage}{0.5\linewidth}
\centering
\begin{tikzpicture}
\begin{axis}[
    width=1.0\linewidth,
    height=0.8\linewidth,
    xlabel={Graph size},
    ylabel={Average solution set size},
    grid=both,
    legend style={font=\scriptsize, at={(0.00,1.18)}, anchor=north west},
    tick label style={font=\scriptsize},
    label style={font=\scriptsize},
]
\addplot[
    color=blue,
    mark=o,
    thick
] coordinates {
    (100, 7.1)
    (200, 8.4)
    (300, 8.7)
    (400, 10.4)
    (500, 11.7)
};
\addlegendentry{\rbexAlg (solution)}
\addplot[
    color=gray,
    mark=square*,
    thick
] coordinates {
    (100, 7.1)
    (200, 8.0)
    (300, 8.4)
    (400, 9.8)
    (500, 11.6)
};
\addlegendentry{\rbapAlg (solution)}
\addplot[
    color=orange,
    mark=triangle*,
    thick
] coordinates {
    (100, 7.0)
    (200, 8.0)
    (300, 8.3)
    (400, 9.6)
    (500, 11.6)
};
\addlegendentry{\rapexAlg (solution)}

\addlegendimage{
    color=gray,
    mark=square*,
    thick,
    dash pattern=on 2pt off 1pt
}
\addlegendentry{\rbapAlg (runtime)}

\addlegendimage{
    color=orange,
    mark=triangle*,
    thick,
    dash pattern=on 2pt off 1pt
}
\addlegendentry{\rapexAlg (runtime)}
\end{axis}

\begin{axis}[
    width=1.0\linewidth,
    height=0.8\linewidth,
    axis x line=none,
    axis y line*=right,
    ylabel={Average runtime (seconds)},
    ymode=log,
    tick label style={font=\scriptsize},
    label style={font=\scriptsize},
]
\addplot[
    color=gray,
    mark=square*,
    thick,
    dash pattern=on 2pt off 1pt
] coordinates {
    (100, 0.03)
    (200, 0.16)
    (300, 0.37)
    (400, 0.88)
    (500, 1.33)
};
\addplot[
    color=orange,
    mark=triangle*,
    thick,
    dash pattern=on 2pt off 1pt
] coordinates {
    (100, 0.005)
    (200, 0.016)
    (300, 0.033)
    (400, 0.058)
    (500, 0.081)
};
\end{axis}
\end{tikzpicture}
\end{minipage}
\vspace{-4mm}
\caption{(Left) A preorder of four rules in \textbf{(S3)}). (Right) The average solution set sizes and runtimes for \rbexAlg, \rbapAlg and \rapexAlg on random graphs with this rulebook under the approximation $\EPS = [0.01, 0.01, 0.01, 0.01]$.}
\label{fig:random_approx_pareto_optimal_r34Equal_hierarchy}
\end{figure}

\vspace{-4.5mm}
\subsection*{Setting (S3): Approximate Planning with Hierarchical Objectives}
\vspace{-2mm}
Lastly, we report the performance of \rapexAlg in the presence of rule hierarchies with non-zero approximation vector $\EPS$. Due to the absence of a prior work done on this, we compared the results against the approximate extension of \rbexAlg, where an additional global dominance check was incorporated. In this variant, which we call \rbapAlg, whenever a new candidate realization $x$ is constructed, this global dominance check examines whether or not the cost of $x$ is approximately dominated by some realization $x'$ already present in the solution set. If so, then the realization $x$ is discarded since the cost of any realization with the prefix $x$ will also be approximately dominated by $x'$. Consequently, \rbapAlg not only finds an $\EPS$-approximate rulebook-optimal solution set but also guarantees that every solution is rulebook-optimal.

For this experiment, we randomly generated graphs of varying sizes with 4 objectives. The rule hierarchy assigns equal importance to $r_3$ and $r_4$, while $r_1$ and $r_2$ have lower priority than both. Specifically, the hierarchy (illustrated in Fig. \ref{fig:random_approx_pareto_optimal_r34Equal_hierarchy}) is given by $r_3 \lesssim r_4$, $r_4 \lesssim r_3$, $r_2 < r_3$, $r_1 < r_4$. Fig. \ref{fig:random_approx_pareto_optimal_r34Equal_hierarchy} shows the average runtimes and the solution sizes for \rbapAlg and \rapexAlg. In all of them, \rapexAlg runs more than an order of magnitude faster than \rbapAlg, with the difference growing more and going up to 16$\times$ the runtime of \rbapAlg on the largest graph. The solution set sizes, however, remain relatively comparable, with \rapexAlg having a slight advantage on average while \rbapAlg being closer to the rulebook-optimal frontier sizes. This makes sense when we consider that the solution set returned by \rbapAlg is also a subset of the rulebook-optimal frontier, and hence might require more solutions from the frontier to cover what a few solutions not on the frontier could approximately cover.

We also evaluate the effect of dimensionality reduction in \rapexAlg by comparing it against a baseline version that operates in the full objective space. As shown in Figure~\ref{fig:rapex_vs_baselineRapex}, dimensionality reduction leads to a substantial decrease in runtime, often by an order of magnitude (up to 17$\times$) for better $\EPS$ factors, while producing solution sets of comparable size. Although the extent of this improvement depends on the specific rule hierarchy, these results demonstrate that exploiting rulebook structure to reduce effective dimensionality can significantly improve computational efficiency without sacrificing solution quality.







\begin{figure}[t]
\centering
\begin{minipage}{0.48\linewidth}
\centering
\begin{tikzpicture}
\begin{loglogaxis}[
    xlabel={Baseline \rapexAlg},
    ylabel={\rapexAlg},
    width=1.0\linewidth,  
    height=0.8\linewidth, 
    xmin=0.1, xmax=350,
    ymin=0.1, ymax=350,
    axis on top,
    legend pos=north west,           
    legend style={font=\scriptsize},
    legend cell align=left,
    xtick={0.1, 1, 10, 100},
    ytick={0.1, 1, 10, 100},
    tick align=inside,
    major tick length=0.15cm,
    minor tick length=0.08cm,
    tick label style={font=\scriptsize},  
    label style={font=\scriptsize},       
]


    \addplot[domain=0.1:350, samples=2, dashed, gray, thick, forget plot] {x}
        node[pos=0.8, above left, yshift=-3pt, black, font=\scriptsize\bfseries\sffamily] {1x};

    \addplot[domain=0.1:350, samples=2, dashed, gray, thick, forget plot] {0.2*x}
        node[pos=0.8, above left, xshift=5pt, black, font=\scriptsize\bfseries\sffamily] {5x};

    \addplot[domain=0.1:350, samples=2, dashed, gray, thick, forget plot] {0.05714*x}
        node[pos=0.83, above right, yshift=-5pt, black, font=\scriptsize\bfseries\sffamily] {17x};


    \addplot[
        only marks,
        mark=*,
        mark size=2.5pt,
        color=yellow!30!black,
        fill=yellow!90!white,
        fill opacity=0.5
    ] coordinates {
        (0.278249, 0.242563)
        (1.39836, 0.92723)
        (0.521632, 0.364693)
        (0.179364, 0.154838)
        (2.31062, 1.43322)
        (0.535222, 0.395524)
        (4.73583, 3.12618)
        (0.096946, 0.082752)
        (0.438637, 0.453319)
        (0.844323, 0.455552)
        (0.027501, 0.031232)
        (1.91061, 1.46438)
        (3.54767, 2.75385)
        (0.076467, 0.072451)
        (0.8337, 0.694022)
        (2.20427, 1.55267)
        (0.628665, 0.413462)
        (0.067561, 0.061324)
        (0.049566, 0.054438)
        (3.46893, 1.49077)
        (0.046773, 0.042855)
        (0.312467, 0.288197)
        (0.447858, 0.423595)
        (0.035775, 0.038798)
        (1.43272, 0.73592)
    };
    \addlegendentry{$\varepsilon$ = 0.01}

    \addplot[
        only marks,
        mark=*,
        mark size=2.5pt,
        color=orange!90!black,
        fill=orange!90!white,
        fill opacity=0.6
    ] coordinates {
        (1.29759, 0.563712)
        (5.32341, 1.56354)
        (1.39163, 0.627631)
        (0.533818, 0.234604)
        (11.3271, 2.72566)
        (1.29711, 0.622449)
        (34.7387, 6.19527)
        (0.259944, 0.157478)
        (2.09487, 0.797591)
        (1.90382, 0.910609)
        (0.059681, 0.050967)
        (5.97376, 2.104)
        (11.3518, 4.85029)
        (0.424031, 0.251112)
        (5.53381, 1.76673)
        (8.86587, 2.97476)
        (2.40319, 0.763425)
        (0.173269, 0.122973)
        (0.1262, 0.104151)
        (23.3802, 3.20727)
        (0.093057, 0.060271)
        (1.67824, 0.523993)
        (1.66248, 0.608766)
        (0.0565, 0.049977)
        (9.44679, 1.63738)
    };
    \addlegendentry{$\varepsilon$ = 0.001}

    \addplot[
        only marks,
        mark=*,
        mark size=2.5pt,
        color=red!10!black,
        fill=red!90!white,
        fill opacity=0.7
    ] coordinates {
        (3.16818, 0.937689)
        (19.6015, 2.46374)
        (2.59193, 0.885928)
        (1.244, 0.309127)
        (35.9699, 5.09753)
        (2.53504, 0.850288)
        (119.013, 12.4646)
        (0.568707, 0.226857)
        (4.8568, 1.29186)
        (8.53443, 1.75535)
        (0.088679, 0.063899)
        (28.129, 3.49878)
        (55.0542, 11.1792)
        (0.830504, 0.345469)
        (10.8242, 2.37179)
        (15.5952, 4.09765)
        (5.94939, 1.10977)
        (0.260515, 0.140369)
        (0.273047, 0.161141)
        (70.6848, 4.34772)
        (0.126622, 0.067245)
        (4.41472, 0.776258)
        (4.68409, 0.937405)
        (0.088932, 0.059495)
        (18.3487, 2.40209)
    };
    \addlegendentry{$\varepsilon$ = 0.0001}

\end{loglogaxis}
\end{tikzpicture}
\end{minipage}
\hfill
\begin{minipage}{0.48\linewidth}
\centering
\begin{tikzpicture}
\begin{loglogaxis}[
    xlabel={Baseline \rapexAlg},
    ylabel={\rapexAlg},
    width=1.0\linewidth,  
    height=0.8\linewidth, 
    xmin=1, xmax=350,
    ymin=1, ymax=350,
    axis on top,
    legend pos=north west,
    legend style={font=\scriptsize},
    legend cell align=left,
    xtick={1, 10, 100},
    ytick={1, 10, 100},
    tick align=inside,
    major tick length=0.15cm,
    minor tick length=0.08cm,
    tick label style={font=\scriptsize},  
    label style={font=\scriptsize},       
]

    \addplot[domain=1:350, samples=2, dashed, gray, thick, forget plot] {x}
        node[pos=0.8, above left, black, font=\scriptsize\bfseries\sffamily] {1x};




    \addplot[
        only marks,
        mark=*,
        mark size=2.5pt,
        color=yellow!30!black,
        fill=yellow!90!white,
        fill opacity=0.5
    ] coordinates {
        (4.0, 4.0)
        (11.0, 14.0)
        (11.0, 6.0)
        (6.0, 8.0)
        (6.0, 7.0)
        (7.0, 7.0)
        (13.0, 20.0)
        (3.0, 2.0)
        (8.0, 8.0)
        (6.0, 4.0)
        (2.0, 2.0)
        (5.0, 5.0)
        (4.0, 4.0)
        (3.0, 2.0)
        (4.0, 4.0)
        (6.0, 10.0)
        (7.0, 7.0)
        (6.0, 4.0)
        (2.0, 3.0)
        (24.0, 25.0)
        (7.0, 6.0)
        (5.0, 5.0)
        (5.0, 7.0)
        (3.0, 3.0)
        (11.0, 10.0)
    };
    \addlegendentry{$\varepsilon$ = 0.01}

    \addplot[
        only marks,
        mark=*,
        mark size=2.5pt,
        color=orange!90!black,
        fill=orange!90!white,
        fill opacity=0.6
    ] coordinates {
        (24.0, 23.0)
        (38.0, 38.0)
        (21.0, 22.0)
        (27.0, 27.0)
        (41.0, 38.0)
        (19.0, 18.0)
        (113.0, 80.0)
        (11.0, 11.0)
        (30.0, 29.0)
        (19.0, 19.0)
        (7.0, 7.0)
        (28.0, 28.0)
        (28.0, 31.0)
        (14.0, 13.0)
        (27.0, 22.0)
        (24.0, 24.0)
        (29.0, 29.0)
        (10.0, 10.0)
        (5.0, 5.0)
        (75.0, 68.0)
        (13.0, 13.0)
        (42.0, 43.0)
        (23.0, 23.0)
        (6.0, 6.0)
        (73.0, 76.0)
    };
    \addlegendentry{$\varepsilon$ = 0.001}

    \addplot[
        only marks,
        mark=*,
        mark size=2.5pt,
        color=red!10!black,
        fill=red!90!white,
        fill opacity=0.7
    ] coordinates {
        (42.0, 42.0)
        (127.0, 127.0)
        (37.0, 37.0)
        (54.0, 54.0)
        (91.0, 91.0)
        (34.0, 34.0)
        (282.0, 275.0)
        (32.0, 32.0)
        (49.0, 49.0)
        (59.0, 60.0)
        (10.0, 10.0)
        (92.0, 92.0)
        (80.0, 80.0)
        (24.0, 24.0)
        (63.0, 60.0)
        (48.0, 47.0)
        (54.0, 54.0)
        (15.0, 15.0)
        (10.0, 10.0)
        (150.0, 150.0)
        (18.0, 18.0)
        (79.0, 79.0)
        (52.0, 52.0)
        (12.0, 12.0)
        (101.0, 101.0)
    };
    \addlegendentry{$\varepsilon$ = 0.0001}

\end{loglogaxis}
\end{tikzpicture}
\end{minipage}
\vspace{-4mm}
\caption{Individual runtimes (left) and solution sizes (right) of the baseline \rapexAlg and \rapexAlg with dimensionality reduction on 4-objective road map instances with 321,270 states and 794,830 edges for varying $\EPS$ approximations with the rule hierarchy being $r_1 \sim r_2$, $r_1>r_3$, $r_2 > r_4$ (similar to Figure \ref{fig:random_approx_pareto_optimal_r34Equal_hierarchy}).}
\label{fig:rapex_vs_baselineRapex}
\end{figure}
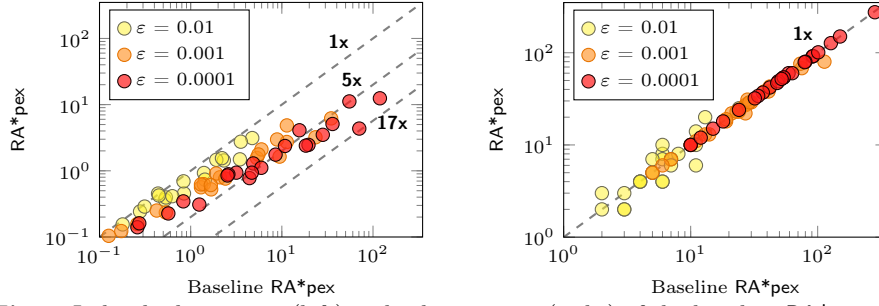

\vspace{-5mm}
\section{Conclusions}
\vspace{-3mm}
We presented \rapexAlg, an approximate multi-objective search algorithm that leverages the rulebook formalism to precisely express complex relationships among objectives, while providing provable approximation guarantees. \rapexAlg generalizes \apexAlg to settings where objectives are partially ordered by priority. A key theoretical contribution of this work is the identification of one-sided transitivity as a crucial structural property for approximate rule-dominance. We showed that the proposed definition of $\EPS$-rule-dominance satisfies one-sided transitivity with respect to exact rule-dominance, which allows dominance relations to be safely propagated across search expansions and establish the correctness of \rapexAlg. 

Our experimental results evaluated \rapexAlg across three problem settings that vary in the presence of objective hierarchy and the use of approximation. Across all settings, \rapexAlg consistently matched or significantly outperformed specialized state-of-the-art algorithms. Overall, \rapexAlg bridges the gap between exact rulebook-based planning and approximate multi-objective search, offering a principled and efficient approach to approximate multi-objective search under rulebooks. We believe this approach opens the door to scalable planning and control synthesis in complex systems where objectives are partially ordered and exact optimality is neither required nor tractable.

\begin{credits}
    This work was supported in part by NSF under Grant
    CNS-2141153.
\end{credits}


%
\bibliographystyle{splncs04}
\bibliography{bibliography}

\end{document}